\documentclass{article}

\usepackage{mylatexstyle}
\usepackage[final]{neurips_2025}

\usepackage[utf8]{inputenc} % allow utf-8 input
\usepackage[T1]{fontenc}    % use 8-bit T1 fonts
\usepackage{hyperref}       % hyperlinks
\usepackage{url}            % simple URL typesetting
\usepackage{booktabs}       % professional-quality tables
\usepackage{amsfonts}       % blackboard math symbols
\usepackage{nicefrac}       % compact symbols for 1/2, etc.
\usepackage{microtype}      % microtypography
\usepackage{xcolor}         % colors
\usepackage{amsmath,amssymb}

\def \poly {\mathrm{poly}}
\def \opt{\mathrm{opt}}

\title{Smoothed Agnostic Learning of Halfspaces over the Hypercube}

\author{%
Yiwen Kou\\
Department of Computer Science, UCLA\\
Los Angeles, CA, US\\
evankou@cs.ucla.edu\\
\and
\textbf{Raghu Meka}\\
Department of Computer Science, UCLA\\
Los Angeles, CA, US\\
raghum@cs.ucla.edu
}

\begin{document}

\maketitle

\begin{abstract}
Agnostic learning of Boolean halfspaces is a fundamental problem in computational learning theory, but it is known to be computationally hard even for weak learning. Recent work \citep{chandrasekaran2024smoothed} proposed smoothed analysis as a way to bypass such hardness, but existing frameworks rely on additive Gaussian perturbations, making them unsuitable for discrete domains. We introduce a new smoothed agnostic learning framework for Boolean inputs, where perturbations are modeled via random bit flips. This defines a natural discrete analogue of smoothed optimality generalizing the Gaussian case. Under strictly subexponential assumptions on the input distribution, we give an efficient algorithm for learning halfspaces in this model, with runtime and sample complexity $\tilde{O}(n^{\poly(\frac{1}{\sigma\epsilon})})$. Previously, such algorithms were known only with strong structural assumptions for the discrete hypercube—for example, independent coordinates or symmetric distributions. Our result provides the first computationally efficient guarantee for smoothed agnostic learning of halfspaces over the Boolean hypercube, bridging the gap between worst-case intractability and practical learnability in discrete settings.
\end{abstract}

\section{Introduction}

Halfspaces, or linear threshold functions (LTFs), are one of the most fundamental concept classes in machine learning. In the realizable setting \citep{10.1145/1968.1972}, they are efficiently learnable by classical algorithms such as the Perceptron \citep{rosenblatt1958perceptron,Novikoff1963ONCP}, Winnow \citep{4568257}, large-margin methods like Support Vector Machines \citep{cortes1995support}, or by linear programming. These methods exploit linear separability and can perform well even in the presence of irrelevant features.

In contrast, the agnostic learning framework \citep{HAUSSLER199278,10.1145/130385.130424}, which allows for arbitrary label noise, poses significant algorithmic challenges. In this setting, the goal is to find a hypothesis that competes with the best in a concept class, without assuming that the data is linearly separable. However, agnostic learning of halfspaces is computationally hard in the worst case: even weak learning—achieving error marginally better than random guessing—is NP-hard under standard complexity assumptions, both in continuous domains \citep{feldman2009agnostic} and on the Boolean hypercube \citep{guruswami2009hardness}.

To overcome these barriers, several restricted models have been studied. For example, under random classification noise (RCN), halfspaces remain learnable using modified Perceptron algorithms or linear programming \citep{Blum1996APA,646140}. Under Massart noise, where adversarial flips are bounded in probability, recent work has led to efficient learning algorithms \citep{awasthi2015efficient,diakonikolas2019distribution,diakonikolas2020learning,diakonikolas2021boosting}. Other lines of work exploit structure in the input distribution: \citet{kalai2008agnostically} gave improper agnostic learning algorithms under uniform, spherical, or log-concave distributions by approximating halfspaces with low-degree polynomials, extending earlier Fourier-based methods \citep{linial1993constant}.

A more recent and promising direction is based on smoothed analysis, which was introduced to explain the practical performance of algorithms that are worst-case hard \citep{spielman2001smoothed}. In learning theory, \citet{chandrasekaran2024smoothed} proposed a smoothed agnostic framework in which the learner competes with the best classifier under slight random perturbations of the inputs. This relaxation enables efficient algorithms for learning low-dimensional concepts, even when worst-case learning is intractable. However, their approach is tailored to continuous domains and relies on additive Gaussian noise, and hence is not suitable for discrete domains such as the Boolean hypercube.

\textbf{Our contribution.} We develop a discrete analogue of smoothed agnostic learning for Boolean concept classes over $\{\pm1\}^n$, where additive Gaussian noise is ill-defined. Instead of perturbing examples in Euclidean space, we introduce bit-flipping noise: each input coordinate is independently flipped with probability $\sigma$. This gives rise to a new benchmark for learning that captures robustness to small discrete perturbations, interpolating between classical agnostic learning ($\sigma = 0$) and random guessing ($\sigma = 1/2$). 

To formalize this, we begin by recalling the standard agnostic learning objective.
\begin{definition}[Agnostic Optimality]\label{def:opt}
Let $\cX$ be a domain and $\cF$ be a class of functions $f:\cX\to\{\pm1\}$. Let $\cD$ be a distribution over labeled examples $(\xb,y)\in\cX\times\{\pm1\}$. The agnostic error of $\cF$ under $\cD$ is defined as
\begin{align*}
    \opt=\inf_{f\in\cF}\PP_{(\xb,y)\sim\cD}[f(\xb)\neq y].
\end{align*}
\end{definition}
We now define a smoothed variant of this benchmark, in which each input is perturbed before evaluation. This definition is general and does not assume any particular structure of the domain or distribution. 

\begin{definition}[Smoothed Optimality]\label{def:smooth opt}
Let $\mathcal{X}$ be a domain and $\mathcal{F}$ be a class of functions $f:\cX\to\{\pm1\}$. Let $\cD$ be a distribution over labeled examples $(\xb,y)\in\cX\times\{\pm1\}$ and $\cP_{\sigma}(\xb)$ be a perturbation distribution of $\xb$ over $\cX$. Define the smoothed agnostic error as:
\begin{align*}
    \opt_{\sigma}=\inf_{f\in\cF}\PP_{(\xb,y)\sim\cD;\tilde{\xb}\sim\mathcal{P}_\sigma(\xb)}[f(\tilde{\xb})\neq y].
\end{align*}
\end{definition}
In standard agnostic learning, the goal is to compete with $\opt$ under $\mathcal{D}$. Following \citet{chandrasekaran2024smoothed}, we instead compete with $\opt_\sigma$ in Definition~\ref{def:smooth opt}.
\begin{definition}[Smoothed Agnostic Learning]\label{def:smoothed_opt}
Fix $\epsilon,\sigma>0$ and $\delta\in(0,1)$. An algorithm $\cA$ learns the class $\cF$ in the $\sigma$-smoothed agnostic setting if, given i.i.d. samples from $\cD$, it outputs a hypothesis $h: \cX \to \{\pm1\}$ such that with probability at least $1 - \delta$:
\begin{align*}
    \PP_{(\xb,y)\sim\cD}[h(\xb)\neq y]\leq\opt_{\sigma}+\epsilon.
\end{align*}
\end{definition}

As an example, \citet{chandrasekaran2024smoothed} study the case where $\mathcal{X} = \mathbb{R}^n$ and the perturbation distribution is additive Gaussian noise: $\mathcal{P}_\sigma(\xb) = \mathcal{N}(\xb, \sigma^2 \Ib_n)$. This formulation relies on the Euclidean structure of $\mathbb{R}^n$ and does not extend to discrete domains. In our setting, we consider $\mathcal{X} = \{\pm1\}^n$ and define perturbations via random bit flips: $\mathcal{P}_\sigma(\xb) = \xb \odot \zb$ where $\zb \sim \mathcal{N}_\sigma$ is a product distribution with $z_i = -1$ with probability $\sigma$ and $1$ otherwise. This defines a natural smoothed learning model for the Boolean hypercube that avoids embedding the domain into $\mathbb{R}^n$. 

Under this framework, we show that halfspaces over the Boolean cube are efficiently learnable in the smoothed agnostic model under mild distributional assumptions. Our approach extends the classical $L_1$-polynomial regression framework to this smoothed setting. The key idea is that every Boolean halfspace, when composed with small random bit-flip perturbations, admits a low-degree polynomial approximation under the input distribution. To establish this, we analyze a smoothed version of the halfspace defined via a noise operator (Definition~\ref{def:operator}), and construct approximators using Berry–Esseen-type arguments combined with critical index analysis to handle irregular weight vectors. We obtain the following result:

\begin{theorem}[Subgaussian-Informal, see also Theorem~\ref{thm:main}]
Let $\cD$ be a distribution on $\{\pm1\}^n\times\{\pm1\}$ with sub-gaussian $\xb$-marginal of variance proxy $\sigma_0^2$. There exists an algorithm that learns the class of linear threshold functions in the $\sigma$-smoothed setting with $N=n^{\poly(\sigma_0/\sigma\epsilon)}\log(1/\delta)$ samples and $\poly(n,N)$ runtime.
\end{theorem}

This is the first result that establishes efficient smoothed agnostic learning of halfspaces over the Boolean hypercube. While our algorithm is improper, it achieves strong generalization guarantees under natural distributions. Previously, such results for the hypercube were only known under very restricted distributions as discussed below. 

\section{Related Work}

\paragraph{Distributional Assumptions in Halfspace Learning:} It is well-understood that agnostically learning halfspaces is intractable in the worst case \citep{feldman2009agnostic,guruswami2009hardness}, even under relatively benign noise models \citep{diakonikolas2022cryptographic}. This has motivated a long line of \emph{distribution-specific} algorithms that guarantee learnability by leveraging assumptions on the data distribution. Early work focused on uniform or product distributions, where powerful Fourier-analytic techniques yield low-degree approximations \citep{linial1993constant,klivans2004learning,blais2010polynomial}. Under the uniform hypercube distribution, halfspace concepts exhibit strong Fourier concentration and low noise sensitivity \citep{o2021analysis}, enabling efficient learning via low-degree polynomial approximation \citep{klivans2004learning}. This was extended to symmetric distributions in \cite{Wimmer10} and to arbitrary product distributions in \cite{blais2010polynomial}. However, beyond these there are very few general classes of distributions over they hypercube where halfspaces are agnostically learnable.

For continuous distributions, halfspaces were shown to be agnostically learnable under log-concave distributions by \cite{kalai2008agnostically} and this was later extended to intersections and other functions of halfspaces in \cite{kane2013learning}. Much like the discrete setting, until the recent work of \cite{chandrasekaran2024smoothed}, most positive results required strong structural assumptions on the marginal distribution of the examples. This work introduced a new smoothed agnostic model which led to several new results for learning halfspaces and functions of halfspaces for a much broader class of distributions (e.g., sub-gaussian or sub-exponential densities).  Our work continues this progression to very general distributions, but focuses on the Boolean domain and shows that only mild tail bounds (strictly sub-exponential) suffice for efficient learning in the smoothed setting.

\textbf{Noise Models and Smoothed Analysis:} In parallel to distributional assumptions on $X$, a complementary line of work has tackled label noise models and smoothed analysis. The classical noise models include random classification noise (RCN), where each label is independently flipped with some probability. \cite{Blum1996APA} gave the first polynomial-time algorithm for learning a halfspace under random classification noise, exploiting the fact that a halfspace’s margin makes it relatively robust to independent label flips. A stronger noise model is the Massart noise model, which bounds the adversary by a flipping probability $\eta<1/2$ on each example. \cite{awasthi2015efficient,pmlr-v49-awasthi16,yan2017revisiting,zhang2017hitting,mangoubi2019nonconvex,diakonikolas2020learning} gave efficient algorithms for learning halfspaces with Massart noise over uniform distribution on the sphere and log-concave distributions. On the other hand, with adversarial (malicious) noise, learning halfspaces requires additional assumptions. \cite{JMLR:v10:klivans09a} designed efficient algorithms for origin-centered halfspaces under malicious noise by assuming isotropic log-concave distribution and small noise rate. In smoothed analysis of learning, one assumes that either the data \citep{10.5555/545381.545499,kane2013learning} or the target concept \citep{chandrasekaran2024smoothed} is randomly perturbed, so that pathological arrangements are avoided. \cite{chandrasekaran2024smoothed} introduced a smoothed agnostic PAC model in $\mathbb{R}^d$ where the learner competes against the best classifier that is robust to slight Gaussian perturbations of examples. Our work can be seen as a Boolean analogue of this idea: rather than perturbing continuous inputs, we require the optimal halfspace to be stable under small random label flips.

\section{Preliminaries}

We review relevant definitions from Boolean function analysis that will allow us to define a discrete smoothing operator and justify using it in place of the original linear threshold function. We use definitions from the analysis of Boolean functions over product spaces, following the framework of \cite{mossel2005noise}. Let $(\Omega_1,\mu_1),\ldots,(\Omega_n,\mu_n)$ be finite probability spaces and let $(\Omega,\mu)$ denote their product. In our setting, we take $\Omega_i = \{\pm1\}$ and define $\mu$ to be the product distribution $\cN_{\sigma}$, where each coordinate is $1$ with probability $1-\sigma$ and $-1$ with probability $\sigma$, independently.

\begin{definition}[$\rho$-noisy copy]\label{def:noisy copy}
Given $\xb\in\Omega$ and $\rho\in[0,1]$, a $\rho$-noisy copy of $\xb$ is a random vector $\yb \sim \cN_{\rho}(\xb)$, where each coordinate $y_i$ is independently set to $x_i$ with probability $\rho$ and to an independent draw from $\mu_i$ with probability $1-\rho$. 
\end{definition}

\begin{definition}[Noise operator $T_\rho$]\label{def:operator}
For any function $f: \Omega \to \RR$ and $\rho \in [0,1]$, the noise operator $T_\rho$ is defined as
\begin{align*}
    (T_{\rho}f)(\xb)=\EE_{\yb\sim\cN_{\rho}(\xb)}[f(\yb)].
\end{align*}
\end{definition}

This definition generalizes the Bonami–Beckner operator \citep{21923} when $\mu$ is the uniform distribution on the hypercube. Intuitively, $T_{\rho}f$ is a smoothed version of $f$, computed by averaging $f$ over a neighborhood of $\xb$ with geometric decay controlled by $\rho$. In particular, $T_1 f = f$, and as $\rho$ decreases from 1, $T_\rho f$ suppresses high-frequency components of $f$. This operator will be used as our main tool for constructing smoothed approximations to Boolean threshold functions.

\begin{definition}[Noise stability and noise sensitivity]
For any $f: \Omega \to \RR$, the noise stability at parameter $\rho$ is defined as
\begin{align*}
    S_{\rho}(f)=\langle f, T_{\rho}f\rangle_{\mu}.
\end{align*}
If $f: \Omega \to \{\pm 1\}$, the noise sensitivity at parameter $\delta \in [0,1]$ is given by
\begin{align*}
    NS_{\delta}(f)=\frac{1}{2}-\frac{1}{2}S_{1-\delta}(f)=\PP_{\xb\sim\mu,\yb\sim\cN_{1-\delta}(\xb)}[f(\xb)\neq f(\yb)].
\end{align*}
\end{definition}

Equivalently, $NS_{\delta}(f) = \Pr_{x,y}[f(x)\neq f(y)]$ where $x$ and $y$ have Hamming correlation $1-2\delta$. This quantity captures the robustness of $f$ to small input perturbations.

It is well-known that natural Boolean functions with low total influence or low-degree Fourier concentration exhibit low noise sensitivity. In particular, linear threshold functions are noise-stable under both uniform \citep{peres2004noise} and general product distributions \citep{blais2010polynomial}. The following lemma bounds the noise sensitivity of halfspaces over arbitrary product spaces. 
\begin{lemma}[Theorem 3.2 in \cite{blais2010polynomial}]\label{lm:ns}
Let $f:\Omega\to\{\pm1\}$ is a linear threshold function, where the domain $\Omega=\Omega_1\times\cdots\times\Omega_n$ has the product distribution $\mu=\mu_1\times\cdots\times\mu_n$. Then $NS_{\delta}(f)\leq\frac{5}{4}\sqrt{\delta}$.
\end{lemma}

This bound implies that for $\rho = 1 - \delta$ close to $1$, the smoothed function $T_\rho f$ closely approximates the original threshold function $f$. This justifies our strategy of working with $T_\rho f$ instead of $f$ in the smoothed learning setting: any learner that performs well on $T_\rho f$ will, up to a small error, also succeed on $f$.

\paragraph{Notation.} We use small boldface characters for vectors and capital bold characters for matrices. We use $[d]$ to denote the set $\{1,2,\cdots,d\}$. For a vector $\xb\in\RR^d$ and $i\in[d]$, $x_i$ denotes the $i$-th coordinate of $\xb$, and $\|\xb\|_2:=\sqrt{\sum_{i=1}^{d}x_i^2}$ the $\ell_2$ norm of $\xb$. For $\xb,\yb\in\RR^d$, we use $\langle\xb,\yb\rangle=\sum_{i=1}^{d}x_iy_i$ as the inner product between them and $\xb\odot\yb=(x_1y_1,\cdots,x_dy_d)$ as the Hadamard product between them. We use $\one\{\cE\}$ to be the indicator function of some event $\cE$. For $(\xb,y)$ distributed according to $\cD$, we denote $\cD_{\xb}$ to be the marginal distribution of $\xb$.

\section{Technical Overview}
In this section, we outline the main steps of our analysis. Our approach follows a reduction-based strategy: we reduce smoothed agnostic learning of Boolean halfspaces to the problem of approximating a smoothed halfspace by a low-degree polynomial, which can then be learned via $L_1$ regression (Section~\ref{sec:l1}). We begin by replacing the original target $f_{\xb}(\zb)=f(\xb\odot\zb)$ with a smoothed surrogate $T_{1-\rho}f_{\xb}(\zb)$ (Definition~\ref{def:operator}), facilitating approximation by low-degree polynomials (Section~\ref{sec:smoothed}).

To handle the biased distribution arising from noise perturbation, we introduce a rerandomization and conditioning trick that rewrites each bit as a mixture involving uniform random variables. This allows us to express the smoothed function as a conditional expectation over uniformly random inputs, making it amenable to quantitative central-limit theorems (Berry–Esseen estimates; Section~\ref{sec:biased}). 
We then use a case analysis facilitated by a  decomposition of the weight vector (Section~\ref{sec:critical}):
\begin{enumerate}
    \item  If a small number of large coordinates (the ``head'') dominate, the halfspace’s output is primarily determined by those coordinates, and we can approximate the function directly. 
    \item Otherwise, the remaining ``tail'' is \emph{regular}, and we apply the Berry–Esseen theorem to approximate the Boolean sum by a Gaussian. This reduces the problem to the continuous setting, where we leverage Gaussian-based techniques (the density ratio method from \cite{chandrasekaran2024smoothed}) to construct low-degree polynomial approximations.
\end{enumerate}
Together, these ingredients yield an efficient smoothed learner for Boolean halfspaces under strictly sub-exponential input distributions.

\subsection{High-Level Approach via \texorpdfstring{$L_1$} Regression}\label{sec:l1}

Our starting point is the $L_1$-polynomial regression method for agnostic learning. In particular, \cite{kalai2008agnostically} established a powerful reduction from agnostic learnability  to low-degree polynomial approximation. 

\begin{algorithm}
\caption{$L_1$ Polynomial Regression Algorithm}\label{alg:l1-regression}
\textbf{Input:} Sample $S = \{(x^1, y^1), \dots, (x^N, y^N)\}$, degree bound $d$
\begin{algorithmic}[1]
\STATE Find polynomial $p$ of degree $\leq d$ to minimize 
\[
\frac{1}{N} \sum_{j=1}^{N} |p(x^j) - y^j|.
\]
(This can be done by expanding examples to include all monomials of degree $\leq d$ and then performing $L_1$ linear regression.)
\STATE Output hypothesis $h(x) = \mathrm{sign}(p(x) - t)$, where $t \in [-1, 1]$ is chosen to minimize the classification error on $S$.
\end{algorithmic}
\end{algorithm}

\begin{theorem}[Theorem 5 in \cite{kalai2008agnostically}]\label{thm:auxi}
Suppose $\min_{\deg(p)\leq d}\EE_{\cD_{\xb}}[|p(\xb)-c(\xb)|]\leq\epsilon$ for some degree $d$ and any $c$ in the concept class $\cC$. Then, for $h$ output by the degree-$d$ $L_1$ polynomial regression algorithm with $N=\poly(n^d/\epsilon)$ examples, $\EE_{S\sim\cD^{N}}[\PP_{(\xb,y)\sim\cD}[h(\xb)\neq y]]\leq\mathrm{opt}+\epsilon$, where $\mathrm{opt}=\min_{f\in\cC}\PP_{(\xb,y)\sim\cD}[f(\xb)\neq y]$. If we repeat the algorithm $r=O(\log(1/\delta)/\epsilon)$ times with fresh examples each, and let $h$ be the hypothesis with lowest error on an independent test set of size $O(\log(1/\delta)/\epsilon^2)$, then with probability at least $1-\delta$, $\PP_{(\xb,y)\sim\cD}[h(\xb)\neq y]\leq\mathrm{opt}+\epsilon$.
\end{theorem}

Theorem~\ref{thm:auxi} says that if the target function $f$ can be approximated in $L_1$ by a low-degree polynomial $p$ with error at most $\epsilon$, then one can efficiently learn $f$ to misclassification error $\mathrm{opt} + \epsilon$, where $\mathrm{opt}$ is the Bayes-optimal error rate under distribution $\cD$. Once such a polynomial is shown to exist, Theorem~\ref{thm:auxi} implies a computationally efficient learning algorithm with sample complexity $N = \poly(n^d/\epsilon) \log(1/\delta)$.

\subsection{Smoothed Learning as Non-Worst-Case Approximation}\label{sec:smoothed}

The challenge is that an arbitrary halfspace $f(\xb)=\mathrm{sign}(\langle\wb,\xb\rangle - \theta)$ might not be well-approximated by any low-degree polynomial over worst-case input distributions. Following \cite{chandrasekaran2024smoothed}, we view smoothed learning as a form of non-worst-case approximation. In this smoothed agnostic setting, the learner’s “effective” target concept is the mapping $(\xb,\zb)\mapsto f(\xb\odot \zb)$, where $\zb \in \{\pm 1\}^n$ is a random noise vector independent of $\xb$ with $\sigma$ close to 0 meaning only a tiny fraction of bits are flipped on average. We extend the $L_1$-regression reduction to handle this scenario. In particular, we prove an analogue of \cite{kalai2008agnostically}’s result tailored to the smoothed model:

\begin{theorem}\label{thm:auxi2}
Suppose $\min_{\deg(p_{\zb})\leq d}\EE_{\zb\sim\cD_{\sigma},\xb\sim\cD_{\xb}}[|p_{\zb}(\xb)-f(\xb\odot\zb)|]\leq\epsilon$ for some degree $d$ and any halfspace $f$, where $\cD_{\xb}$ is any distribution on $\{\pm1\}^n$. Then, for $h$ output by the degree-$d$ $L_1$ polynomial regression algorithm with $N=\poly(n^d/\epsilon)$ examples, $\EE_{S\sim\cD^{N}}[\PP_{(\xb,y)\sim\cD}[h(\xb)\neq y]]\leq\mathrm{opt}_{\sigma}+\epsilon$. If we repeat the algorithm $r=O(\log(1/\delta)/\epsilon)$ times with fresh examples each, and let $h$ be the hypothesis with lowest error on an independent test set of size $O(\log(1/\delta)/\epsilon^2)$, then with probability at least $1-\delta$, $\PP_{(\xb,y)\sim\cD}[h(\xb)\neq y]\leq\mathrm{opt}_{\sigma}+\epsilon$.
\end{theorem}

After this reduction, our task reduces to a purely approximation-theoretic problem: we need to construct, for each noise vector $\zb$, a polynomial $p_{\zb}(\xb)$ in the variable $\xb$ such that the expected $L_1$ error over the smoothing process remains small:
\begin{align*}
    \EE_{\zb\sim\cN_{\sigma},\xb\sim D_{\xb}}\big[|p_{\zb}(\xb)-f(\xb\odot\zb)|\big]\leq\epsilon.
\end{align*}

To achieve this, we treat the smoothing noise $\zb$ and consider $\xb$ as a fixed parameter. This reduces the problem to approximating the function $f_{\xb}(\zb)=f(\xb\odot\zb)$. We replace $f_{\xb}(\zb)$ with its smooth approximation by applying the generalized Bonami-Beckner operator (Definition~\ref{def:operator}) on $\zb$:
\begin{align*}
    T_{1-\rho}f_{\xb}(\zb)=\EE_{\yb\sim\cN_{1-\rho}(\zb)}[f_{\xb}(\yb)].
\end{align*}
Applying Lemma~\ref{lm:ns} with $\rho=O(\epsilon^2)$, we obtain:
\begin{align*}
    \EE_{\zb\sim\cN_{\sigma}}[|T_{1-\rho}f_{\xb}(\zb)-f_{\xb}(\zb)|]\leq\epsilon.
\end{align*}
Therefore, if we can find a low-degree polynomial that approximates $T_{1-\rho}f_{\xb}(\zb)$ well in $L_1$, that polynomial will also succeed in approximating $f_{\xb}(\zb)$. The remainder of our technical approach will be devoted to constructing such a polynomial approximator for the smoothed halfspace $T_{1-\rho}f_{\xb}(\zb)$.

\subsection{From Biased to Uniform Distribution on the Hypercube}\label{sec:biased}
To construct low-degree polynomial approximations, we analyze the noise-smoothed function $T_{1-\rho}f_{\xb}$. Recall that for a fixed input $\xb$, we define $f_{\xb}(\zb)=f(\xb\odot\zb)$, and suppose $f(\cdot)=\sign(\langle\wb,\cdot\rangle-\theta)$. Then we have:
\begin{align*}
    T_{1-\rho}f_{\xb}(\zb)&=\EE_{\yb\sim\cN_{1-\rho}(\zb)}[\sign(\langle\wb\odot\xb,\yb\rangle-\theta)]=\EE_{\yb\sim\cN_{1-\rho}(\zb)}[\sign(\langle\ub,\yb\rangle-\theta)],
\end{align*}
where we define $\ub=\wb\odot\xb$. 

Here, $\zb\sim\cN_{\sigma}$ denotes a product distribution over $\{\pm1\}^n$ where each bit $z_i$ is 1 with probability $1-\sigma$ and $-1$ with probability $\sigma$. The vector $\yb$ is a $(1-\rho)$-noisy copy of $\zb$ (Definition~\ref{def:noisy copy}) with probability $1-\rho$, $y_i=z_i$; otherwise, $y_i$ is redrawn independently from $\cN_{\sigma}$ with probability $\rho$. Therefore, $\yb\sim\cN_{\sigma}$, correlated with $\zb$, follows a biased distribution on the hypercube. To facilitate polynomial approximation, we aim to reduce this to a form where the randomness comes from a uniform distribution. To achieve this, we introduce a rerandomization trick that rewrites each coordinate $y_i$ as:
\begin{equation*}
y_i=(1-l_i)z_i+l_i\tau_i=(1-l_i)z_i+l_i(1-m_i)+l_im_i\epsilon_i,
\end{equation*}
where 
\begin{equation*}
l_i=\begin{cases}
    1\text{ w.p. }\rho\\
    0\text{ w.p. }1-\rho
\end{cases},
\tau_i=\begin{cases}
    1\text{ w.p. }1-\sigma\\
    -1\text{ w.p. }\sigma
\end{cases},
m_i=\begin{cases}
    1\text{ w.p. }2\sigma\\
    0\text{ w.p. }1-2\sigma
\end{cases},
\end{equation*}
with $\epsilon_i$ being a Radmacher random variable (uniform over $\{\pm1\})$. 

This decomposition captures the full noise process: $l_i$ is an indicator that determines whether the coordinate is kept as $z_i$ (with probability $1-\rho$) or resampled as $\tau_i\sim(\cN_{\sigma})_i$ (with probability $\rho$). The variable $m_i$ is then used to rerandomize $\tau_i$, since $\tau_i$ can be viewed as taking the value 1 with probability $1-2\sigma$ (when $m_i=0)$ or a uniform random bit $\epsilon_i$ with probability $2\sigma$ (when $m_i=1$).

A key benefit is that, conditional on $\lb$ and $\mb$, the random component $\bepsilon$ follows the uniform distribution on $\{\pm1\}^n$. We now condition on $(\lb,\mb)$ and express the smoothed function as: 
\begin{align*}
    T_{1-\rho}f_{\xb}(\zb)&=\EE_{\lb,\mb}\Big[\EE_{\yb}[\sign(\langle\ub,\yb\rangle-\theta)|\lb,\mb]\Big]=\EE_{\lb,\mb}\Big[\EE_{\bepsilon}[\sign(\langle\ub,\lb\odot\mb\odot\bepsilon\rangle+b-\theta)|\lb,\mb]\Big],
\end{align*}
where $b$ is a deterministic shift depending on the coordinates fixed by $\lb,\mb$.

Given that $\bepsilon$ is uniform distribution on hypercube, the inner sum behaves like a sum of independent $\{\pm1\}$ random variables. Under mild regularity condition (Definition~\ref{def:reg}) on the weight vector $\ub$, we can apply the Berry–Esseen Theorem to approximate this inner distribution by a Gaussian. Specifically, we approximate: 
\begin{align}
    \langle\ub,\lb\odot\mb\odot\bepsilon\rangle\approx\cN(0,\|\ub\odot\lb\odot\mb\|_2^2). \label{approx}
\end{align}
Substituting into the earlier expression yields the Gaussian-smoothed approximation:
\begin{align*}
    \tilde{T_{1-\rho}f_{\xb}(\zb)}=\EE_{\lb,\mb}\Big[\EE_{s\sim\cN(0,\|\ub\odot\lb\odot\mb\|_2^2)}[\sign(s+b-\theta)|\lb,\mb]\Big]
\end{align*}
This reduces our setting to the Gaussian noise model analyzed in \cite{chandrasekaran2024smoothed} for which efficient low-degree polynomial approximations are known. In particular, the density ratio method developed in that work can be applied to approximate $\tilde{T_{1-\rho}f_{\xb}(\zb)}$ with a small $L_1$ error. 

\subsection{Handling Irregularity via Critical Index Analysis}\label{sec:critical}
Recall that the approximation in \eqref{approx} relies on the Berry–Esseen Theorem, which introduces a uniform approximation error of $O\big((\frac{\|\ub\odot\lb\odot\mb\|_3}{\|\ub\odot\lb\odot\mb\|_2})^3\big)$ for the cumulative density function. This can be further bounded by $O\big(\frac{\|\ub\odot\lb\odot\mb\|_{\infty}}{\|\ub\odot\lb\odot\mb\|_2}\big)$. Note that each coordinate $l_im_i$ is equal to 1 with probability $2\rho\sigma$ and 0 otherwise. By concentration, we have $\|\ub\odot\lb\odot\mb\|_2\approx(2\rho\sigma)^{1/2}\|\ub\|_2$, so the approximation error becomes $O\big((\rho\sigma)^{-1/2}\frac{\|\ub\|_{\infty}}{\|\ub\|_2}\big)$. This motivates the following regularity condition:

\begin{definition}[regularity]\label{def:reg}
For vector $\wb\in\RR^n$, $\wb$ is $\alpha$-regular if $\|\wb\|_{\infty}\leq\alpha\cdot\|\wb\|_2$.
\end{definition}

Given this definition, we see that if $\ub$ is $\alpha$-regular, then the approximation in \eqref{approx} holds with $L_{\infty}$ error $O((\rho\sigma)^{-1/2}\alpha)$. Since $\ub=\wb\odot\xb$ and $\xb\in\{\pm1\}^n$, the regularity of $\ub$ is equivalent to that of $\wb$. For such ``good'' (i.e., $\alpha$-regular with small $\alpha$) weight vectors $\wb$, we can construct low-degree polynomial approximators by reducing to the Gaussian setting analyzed in Section~\ref{sec:biased} and \cite{chandrasekaran2024smoothed}. 

However, we must also handle the ``bad'' or irregular cases, where $\langle\ub,\lb\odot\mb\odot\bepsilon\rangle$ deviates significantly from Gaussian behavior. To deal with such irregular $\wb$, we employ critical index analysis, a standard tool in the analysis of Boolean halfspaces \citep{1663723,doi:10.1137/070707890,diakonikolas2010bounded,meka2010pseudorandom,doi:10.1137/090756466,diakonikolas2013improved}. 
%in the proof of Theorem 5.3 in \cite{diakonikolas2010bounded} and Theorem 4.8 in \cite{meka2010pseudorandom}). 

\begin{definition}[$\alpha$-critical index]
For $\ub\in\RR^n$, assume that $|u_1|\geq\cdots\geq|u_n|$. We define the $\alpha$-critical index $\ell(\alpha)$ of a halfspace $h(\xb)=\sign(\langle\ub,\xb\rangle-\theta)$ as the smallest index $i\in[n]$ for which $|u_i|\leq\alpha\cdot\sigma_i$, where $\sigma_i:=\sqrt{\sum_{j=i}^{n}u_j^2}$.
\end{definition}

Intuitively, the $\alpha$-critical index is the first index $i$ such that the tail weight vector $(u_i,\cdots,u_n)$ is $\alpha$-regular. Our earlier argument covers the case $i=1$, where the entire vector is regular. Using this framework, we obtain the following structural result:

\begin{lemma}[Critical Index Decomposition]\label{lm:approx}
Without loss of generality, let $\ub=\wb\odot\xb$ with entries sorted in non-increasing magnitude, i.e., $|u_1|\geq\cdots\geq|u_n|$. Suppose $\xb$ follows a $(\alpha,\lambda)$-strictly sub-exponential distribution on $\{\pm1\}^n$. For any fixed $\zb$, there exists a threshold $K=K(\alpha,\epsilon)=O\big(\log(1+\lambda)/\alpha^2+\log(1/\epsilon)\log(1/\alpha)/\rho\sigma\alpha^2\big)$ such that one of the following two conditions holds:
\begin{enumerate}
    \item For some $H< K$, the tail vector $\ub_{T}=(u_{H+1},\cdots,u_n)$ is $\alpha$-regular, where $\alpha$ is to be choosen later.
    \item For $H=K$ and at least $1-\epsilon$ fraction of $\xb$, it holds that \begin{align}
    \PP_{\yb\sim\cN_{1-\rho}(\zb)}[\sign(\langle\ub_{H},\yb_{H}\rangle+\langle\ub_{T},\yb_{T}\rangle-\theta)\neq\sign(\langle\ub_{H},\yb_{H}\rangle-\theta)]\leq\epsilon,\label{ineq:approx}
\end{align}
where $\ub_{H}:=(u_1,\cdots,u_H)$.
\end{enumerate}
\end{lemma}

This lemma is proved by analyzing two cases, depending on whether the critical index $\ell(\alpha)$ satisfies $1<\ell(\alpha)<K$, or $\ell(\alpha)\geq K$. In the former case, we set $H=\ell(\alpha)-1$, and $\ub_T$ is $\alpha$-regular. In the latter case, the head vector forms a sufficiently long geometrically decaying sequence $|u_1|\geq\cdots\geq|u_H|$ to ensure that the influence of the remaining tail vector $\ub_T$ on the halfspace output is negligible. That is, with high probability, 
$\sign(\langle\ub,\yb\rangle-\theta)\approx\sign(\langle\ub_H,\yb_H\rangle-\theta)$.

We now show how to construct low-degree polynomial approximators in both cases. 

\textbf{Case 1:} When $\ub_T$ is $\alpha$-regular, we condition on $\yb_H$. For each fixed $\yb_H$, the function becomes a regular halfspace in $\yb_T$:
\begin{align*}
    \tilde{f_{\xb}}(\yb_{T})=\sign(\langle\ub_{T},\yb_{T}\rangle-\tilde{\theta}),\text{ where }\tilde{\theta}=\theta-\langle\ub_H,\yb_H\rangle.
\end{align*}
We apply the techniques of \cite{chandrasekaran2024smoothed} to approximate this with a low-degree polynomial. One subtlety is that directly applying their construction leads to a degree polynomial in $|\tilde{\theta}|/\|\ub_T\|_2$. To address this, we use an indicator trick to define:
\begin{align*}
    p_{\yb_{H}}(\xb)&=\sign(\langle\ub_{H},\yb_{H}\rangle-\theta)\cdot\ind\big(|\langle\ub_{H},\yb_{H}\rangle-\theta|>C\cdot\|\ub_{T}\|_2\big)\\
    &\qquad+\tilde{p}_{\yb_{H}}(\xb)\cdot\ind\big(|\langle\ub_{H},\yb_{H}\rangle-\theta|\leq C\cdot\|\ub_{T}\|_2\big),
\end{align*}
where $\tilde{p}_{\yb_{H}}(\xb)$ can be constructed using the idea from \cite{chandrasekaran2024smoothed} since $|\tilde{\theta}|/\|\ub_T\|_2$ is controlled. The indicator functions are low-degree polynomials of degree at most $H$, since they only depend on $H$ variables and any function $f:\{\pm1\}^k\to\RR$ can be represented by a degree at most $k$ multilinear polynomial.

\textbf{Case 2:} If the second condition of the lemma holds, we approximate $T_{1-\rho}f_{\xb}(\zb)$ directly using $\sign(\langle\ub_H,\yb_H\rangle-\theta)$. Since this depends only on the first $H$ coordinates, it can be exactly represented as a polynomial of degree at most $H$. 

In either case, we obtain a low-degree polynomial approximator for the smoothed function $T_{1-\rho}f_{\xb}(\zb)$.

\subsection{Results}\label{sec:results}
Using this framework, we establish the following approximation bound:

\begin{definition}[Strictly Sub-exponential Distributions]
A distribution $\cD$ on $\RR^d$ is $(\alpha,\lambda)$-strictly sub-exponential if for all $\|\vb\|_2=1$, $\PP_{\xb\sim\cD}[|\langle\xb,\vb\rangle|>t]\leq 2\cdot e^{-(t/\lambda)^{1+\alpha}}$.
\end{definition}

\begin{lemma}\label{lm:approx3}
Fix $\epsilon>0$ and a sufficiently large universal constant $C>0$. Let $\cD$ be a $(\alpha,\lambda)$-strictly sub-exponential distribution on $\{\pm1\}^{n}$. Let $f:\{\pm1\}^n\to\{\pm1\}$ be a linear threshold function. There exists a family of polynomials $p_{\zb}$ parameterized by $\zb$ of degree at most $O\Big(\big(C\sigma^{-\frac{1}{2}}\lambda\log(1/\epsilon)/\epsilon\big)^{6(1+\frac{1}{\alpha})^3}\Big)$ such that $\EE_{\zb\sim\cN_{\sigma}}\EE_{\xb\sim D}[|p_{\zb}(\xb)-f_{\xb}(\zb)|]$ is at most $\epsilon$.
\end{lemma}

Given the polynomial approximation and the degree upper bound, one can directly run $L_1$ polynomial regression (Algorithm~\ref{alg:l1-regression}) as stated in Theorem~\ref{thm:auxi2}. We now can get our main theorem for strictly sub-exponential distributions.

\begin{theorem}\label{thm:main}
Let $\cD$ be a distribution on $\{\pm1\}^n\times\{\pm1\}$ such that the marginal distribution is $(\alpha,\lambda)$-strictly sub-exponential. There exists an algorithm that draws $N=n^{\poly((\lambda/\sigma\epsilon)^{(1+1/\alpha)^3})}\log(1/\delta)$ samples, runs in time $\poly(n,N)$, and computes a hypothesis $h(\xb)$ such that, with probability at least $1-\delta$, it holds that $\PP_{(\xb,y)\sim \cD}[y\neq h(\xb)]\leq\mathrm{opt}_{\sigma}+\epsilon$.
\end{theorem}

Our main theorem shows that any Boolean halfspace on $\{\pm1\}^n$ can be learned agnostically in the smoothed model under strictly sub-exponential input distributions. This result holds in a general and challenging setting where prior techniques fail, and it achieves efficient runtime and sample complexity. Table~\ref{tab:comparison} compares our guarantees with the most relevant prior works. Conceptually, our contributions extend the scope of agnostic halfspace learning in two fundamental directions:

\textbf{Relaxing distributional assumptions via smoothed optimality:} A key technical contribution is our use of a smoothed benchmark $\opt_\sigma$ (Definition~\ref{def:smoothed_opt}) instead of the worst-case error $\opt$ (Definition~\ref{def:opt}), enabling learning under substantially weaker distributional assumptions. In particular, we show that halfspaces remain efficiently learnable under general strictly sub-exponential marginals, which is a significant relaxation compared to the strong structural assumptions required in earlier work. For example, the Fourier-based techniques of \cite{KLIVANS2004808,kalai2008agnostically,blais2010polynomial} exploit spectral concentration under uniform or product distributions to obtain low-degree polynomial approximations. To go beyond the product setting, \cite{Wimmer10} generalized previous techniques to  symmetric group to handle permutation-invariant distributions. However, these methods break down when the input has more dependencies or heavier tails. In contrast, our approach succeeds under strictly sub-exponential marginals by combining bit-flip smoothing with a critical index decomposition and Berry–Esseen approximation, enabling polynomial approximation without requiring coordinate independence or permutation invariant structure.

\textbf{Extending the smoothed learning framework to the Boolean hypercube:} A second core contribution is our extension of smoothed agnostic learning to the Boolean domain $\{\pm1\}^n$, where additive Gaussian perturbations used in prior smoothed models are not well-defined. In continuous domains, several tools, including Gaussian surface area bounds \citep{Klivans2008LearningGC}, log-concave concentration inequalities \citep{kane2013learning}, and Gaussian smoothing combined with density ratio techniques \citep{chandrasekaran2024smoothed}, enable efficient agnostic learning of halfspaces. However, none of these directly apply to the hypercube. Our analysis circumvents this barrier by performing a case analysis based on the critical index of the weight vector: either a small number of head coordinates dominate and effectively determine the output, or the remaining tail is regular, allowing us to invoke the Berry–Esseen theorem to approximate the Boolean tail sum by a Gaussian, thereby enabling the use of continuous tools developed in prior work \citep{chandrasekaran2024smoothed}.

\begin{table}[t]
\centering
\small
\caption{Comparison of $\mathrm{opt}/\mathrm{opt}_{\sigma}+\epsilon$ agnostic-learning guarantees for halfspaces (without structural noise assumptions). We ignore the polynomial logarithmic factors in $1/\delta$.}
\label{tab:comparison}
\begin{tabular}{@{}lccccc@{}}
\toprule
Work & Domain & Distribution & Bench. & Smooth & Complexity \\ \midrule
\citet{kalai2008agnostically} & $\{\pm1\}^n$ & Uniform & $\opt$ & None & $n^{O(\frac{1}{\epsilon^4})}$ \\
\citet{kalai2008agnostically} & $S^{n-1}$ & Uniform & $\opt$ & None & $n^{O(\frac{1}{\epsilon^4})}$ \\
\citet{kalai2008agnostically} & $\RR^n$ & Log-concave & $\opt$ & None & $n^{O(d(\epsilon))}$ \\
\citet{Klivans2008LearningGC} & $\mathbb{R}^n$ & Gaussian & $\opt$ & None & $n^{O(\frac{1}{\epsilon^4})}$ \\
\cite{Wimmer10} & $[B]^n$ & Perm-Inv & $\opt$ & None & $n^{O(\frac{1}{\epsilon^4})}$  \\
\citet{kane2013learning} & $\mathbb{R}^n$ & Sub-exp & $\opt_\sigma$ & Input noise & $n^{\exp(\frac{(\log\log(\frac{1}{\sigma\epsilon}))^{\tilde{O}(1)}}{\sigma^4\epsilon^4})}$ \\
\citet{chandrasekaran2024smoothed} & $\mathbb{R}^n$ & Strictly Sub-exp & $\opt_\sigma$ & Concept noise & $n^{\poly((\frac{\lambda}{\sigma\epsilon})^{(1+\frac{1}{\alpha})^3})}$ \\
\textbf{Ours} & $\{\pm1\}^n$ & Strictly Sub-exp & $\opt_\sigma$ & Concept noise & $n^{\poly((\frac{\lambda}{\sigma\epsilon})^{(1+\frac{1}{\alpha})^3})}$ \\
\bottomrule
\end{tabular}
\end{table}

\section{Conclusion and Open Problems}\label{sec:conclusion}
In this work, we extended the smoothed agnostic learning framework to the Boolean hypercube, and demonstrated that halfspaces are efficiently learnable with respect to a broad class of input distributions (strictly sub-exponential marginals). Our approach combines tools from smoothing analysis, conditional polynomial approximation, and critical index decomposition to construct low-degree polynomial approximators in a discrete setting where standard analytic techniques are not applicable. By competing with a smoothed benchmark $\mathrm{opt}_{\sigma}$, our guarantee circumvents known hardness results for agnostic learning over the hypercube, while matching the sample and runtime complexity of prior work in continuous domains.

Our current techniques apply only to single halfspaces, and the polynomial degree and runtime degrade as the smoothing parameter $\sigma$ becomes small. In addition, our analysis requires strictly sub-exponential tail assumptions, and it remains unclear whether comparable guarantees are achievable under weaker conditions. Our results also suggest a potential link to agnostic learning under smoothed input distributions, analogous to the Gaussian framework in continuous domains \citep{kalai2008decision,10.5555/1747597.1748053,kane2013learning,chandrasekaran2024smoothed}. Formalizing this connection in the Boolean setting appears subtle, due to the lack of Euclidean geometry and the discrete nature of bit-flip noise, and we leave it as an intriguing direction for future work.

An important open question is whether these techniques can be extended to intersections of multiple halfspaces. While our framework theoretically supports such generalizations under smoothed optimality, a major technical challenge arises in adapting critical index analysis to this setting. For a single halfspace, sorting the coordinates of the weight vector by magnitude plays a central role in identifying regular and irregular components. However, in the case of multiple halfspaces, each weight vector may induce a different ordering over coordinates, making it difficult to define a unified notion of ``head'' and ``tail'' variables. As a result, applying a shared conditioning or decomposition strategy becomes nontrivial. Developing new structural insights or approximation techniques that can handle this multi-directional irregularity remains an open problem.

More broadly, this raises the question of how far the smoothed learning framework can be pushed. Can it yield efficient algorithms for learning other complex Boolean concept classes (e.g., DNF formulas, decision lists, or polynomial threshold functions of higher degree) under heavy-tailed distributions? Can it be made adaptive to unknown noise levels or to distributions that do not satisfy strict tail bounds? We leave these questions for future work.

\begin{ack}
We sincerely thank the anonymous reviewers for their helpful comments. The authors acknowledge support in part from the National Science Foundation under Award CCF-2217033 (EnCORE: Institute for Emerging CORE Methods in Data Science).
\end{ack}

\bibliography{deeplearningreference}
\bibliographystyle{ims}

% \section*{References}

% References follow the acknowledgments in the camera-ready paper. Use unnumbered first-level heading for
% the references. Any choice of citation style is acceptable as long as you are
% consistent. It is permissible to reduce the font size to \verb+small+ (9 point)
% when listing the references.
% Note that the Reference section does not count towards the page limit.
% \medskip

% {
% \small

% [1] Alexander, J.A.\ \& Mozer, M.C.\ (1995) Template-based algorithms for
% connectionist rule extraction. In G.\ Tesauro, D.S.\ Touretzky and T.K.\ Leen
% (eds.), {\it Advances in Neural Information Processing Systems 7},
% pp.\ 609--616. Cambridge, MA: MIT Press.

% [2] Bower, J.M.\ \& Beeman, D.\ (1995) {\it The Book of GENESIS: Exploring
%   Realistic Neural Models with the GEneral NEural SImulation System.}  New York:
% TELOS/Springer--Verlag.

% [3] Hasselmo, M.E., Schnell, E.\ \& Barkai, E.\ (1995) Dynamics of learning and
% recall at excitatory recurrent synapses and cholinergic modulation in rat
% hippocampal region CA3. {\it Journal of Neuroscience} {\bf 15}(7):5249-5262.
% }

%%%%%%%%%%%%%%%%%%%%%%%%%%%%%%%%%%%%%%%%%%%%%%%%%%%%%%%%%%%%

\newpage
\appendix

\section{Bonami-Beckner Operator Approximation}
We show that for any linear threshold function $f:\{\pm1\}^n\to\{\pm1\}$ the approximation error $L_1$ of the operator $T_{1-\rho}f$ to $f$ can be upper bounded by $O(\sqrt{\rho})$.
\begin{lemma}
For any linear threshold function $f:\{\pm1\}^n\to\{\pm1\}$ and $\sigma,\rho\in[0,1]$, it holds that
\begin{align*}
    \EE_{\zb\sim\cN_{\sigma}}[|T_{1-\rho}f(\zb)-f(\zb)|]\leq\frac{5}{2}\sqrt{\rho}.
\end{align*}
\end{lemma}

\begin{proof}
By triangle inequality and special case of Lemma~\ref{lm:ns} when $\Omega=\{\pm1\}^n$ and $\mu=\cN_{\sigma}$, we have
\begin{align*}
    \EE_{\zb\sim\cN_{\sigma}}[|T_{1-\rho}f(\zb)-f(\zb)|]&=\EE_{\zb\sim\cN_{\sigma}}\big[\big|\EE_{\yb\sim\cN_{1-\rho}(\zb)}[f(\yb)]-f(\zb)\big|\big]\\
    &\leq\EE_{\zb\sim\cN_{\sigma}}\big[\EE_{\yb\sim\cN_{1-\rho}(\zb)}[|f(\yb)-f(\zb)|\big]\big]\\
    &=2\EE_{\zb\sim\cN_{\sigma},\yb\sim\cN_{1-\rho}(\zb)}\big[\ind[f(\yb)\neq f(\zb)]\big]\\
    &\leq 2NS_{\rho}(f)\\
    &\leq\frac{5}{2}\sqrt{\rho}.
\end{align*}
\end{proof}

Therefore, choosing $\rho=O(\epsilon^2)$ makes this error at most $\epsilon/2$.

\section{Polynomial Approximation for \texorpdfstring{$T_{1-\rho}f_{\xb}(\zb)$}{}}

We now approximate $T_{1-\rho}f_{\xb}(\zb)$ using a polynomial for the more general class of strictly sub-exponential distributions.
\begin{definition}[Strictly Sub-exponential Distributions]
A distribution $\cD$ on $\RR^d$ is $(\alpha,\lambda)$-strictly sub-exponential if for all $\|\vb\|_2=1$, $\PP_{\xb\sim\cD}[|\langle\xb,\vb\rangle|>t]\leq 2\cdot e^{-(t/\lambda)^{1+\alpha}}$.
\end{definition}

Our main goal in this section is to prove the following polynomial approximation result in Lemma~\ref{lm:approx3}.

Suppose $f(\xb)=\sign(\langle\wb,\xb\rangle-\theta)$. Denote $\wb\odot\xb$ as $\ub$. Without loss of generality, suppose that $|u_1|\geq|u_2|\geq\cdots\geq|u_n|$. Then, we have
\begin{align*}
    T_{1-\rho}f_{\xb}(\zb)&=\EE_{\yb\sim\cN_{1-\rho}(\zb)}[\sign(\langle\wb\odot\xb,\yb\rangle-\theta)]=\EE_{\yb\sim\cN_{1-\rho}(\zb)}[\sign(\langle\ub,\yb\rangle-\theta)].
\end{align*}

To obtain a polynomial approximation of $T_{\rho}f_{\xb}(\zb)$, we first prove Lemma~\ref{lm:approx}.

\begin{proof}
Let
\begin{align*}
    L(\alpha,\epsilon)&=\lceil\log(1/\epsilon)/\rho\sigma\rceil\cdot\lceil(4/\alpha^2)\log(1/\alpha)\rceil+2\log(C/\alpha)/\alpha^2\\
    &=O\big(\log(1/\epsilon)\log(1/\alpha)/\rho\sigma\alpha^2\big)+O\big(\log(\max\{\lambda,1\}\log^{\max\{\frac{1}{2},\frac{1}{1+\alpha}\}}(1/\epsilon)/\alpha)/\alpha^2\big)\\
    &=O\big(\log(1/\epsilon)\log(1/\alpha)/\rho\sigma\alpha^2\big)+O\big(\log(\max\{\lambda,1\}\log(1/\epsilon)/\alpha)/\alpha^2\big)\\
    &=O\big(\log(1+\lambda)/\alpha^2+\log(1/\epsilon)\log(1/\alpha)/\rho\sigma\alpha^2\big),
\end{align*}
where $C=(1-2\rho\sigma)\cdot\lambda\log^{\frac{1}{1+\alpha}}(2/\epsilon)+\sqrt{2\log(2/\epsilon)}$.

If $\alpha$-critical index $\ell(\alpha)<L(\alpha,\epsilon)$, then the first condition holds by taking $H=\ell(\alpha)-1$. If the $\alpha$-critical index $\ell(\epsilon)\geq L(\alpha,\epsilon)$, we will show that the second condition holds. 

By Lemma 5.5 in \cite{diakonikolas2010bounded}, there exist a set of nicely separated coordinates $G=\{i_1,i_2,\cdots,i_t\}\subseteq H$ where $i_1<i_2<\cdots<i_t$ and $i_{k+1}-i_{k}=\lceil4\log(1/\alpha)/\alpha^2\rceil$ such that $|u_{i_{k+1}}|\leq|u_{i_k}|/3$ for any $k\in [t-1]$. Then by Claim 5.7 in \cite{diakonikolas2010bounded}, for any two points $\xb_1\neq\xb_2\in\{\pm1\}^t$, we have $|\langle\ub_{G},\xb_1\rangle-\langle\ub_{G},\xb_2\rangle|\geq |u_{i_t}|$. Take $t=\lceil\log(1/\epsilon)/\rho\sigma\rceil$. For any fixed assignment to the variables in $H\setminus G$, we have
\begin{align*}
    &\PP_{\yb}\bigg[\Big|\sum_{i\in H}u_iy_i-\theta\Big|\leq\frac{|u_{i_t}|}{4}\bigg]\\
    &=\PP_{\yb}\bigg[\sum_{i\in G}u_iy_i\in\Big[\theta-\sum_{i\in H\setminus G}u_iy_i-\frac{|u_{i_t}|}{4},\theta-\sum_{i\in H\setminus G}u_iy_i+\frac{|u_{i_t}|}{4}\Big]\bigg]\\
    &\leq\max_{\xb_1\in\{\pm1\}^t}\PP_{\yb_{G}}[\yb_{G}=\xb_1]\\
    &\leq(1-\rho\sigma)^{t}\leq e^{-\rho\sigma t}\leq\epsilon,
\end{align*}
where the second inequality is because there's at most one point in an interval of length $|u_{i_t}|$ given that $\langle\ub_{G},\xb_1\rangle$ are well-separated. The third inequality is because
\begin{align*}
    &\max\{\PP[y_i=1],\PP[y_i=-1]\}\leq\max\{1-\rho\sigma,\rho\sigma\}=1-\rho\sigma,&\text{ if }z_i=1,\\
    &\max\{\PP[y_i=1],\PP[y_i=-1]\}\leq\max\{\rho(1-\sigma),1-\rho(1-\sigma)\}\leq 1-\rho\sigma,&\text{ if }z_i=-1.
\end{align*}
By our choice of $L(\alpha,\epsilon),t,i_t$, we have
\begin{align*}
    L(\alpha,\epsilon)-i_t\geq L(\alpha,\epsilon)-\lceil\log(1/\epsilon)/\rho\sigma\rceil\cdot\lceil(4/\alpha^2)\log(1/\alpha)\rceil\geq 2\log(C/\alpha)/\alpha^2.
\end{align*}
By applying Lemma 5.5 in \cite{diakonikolas2010bounded}, we have
\begin{align*}
    \|\ub_T\|_2=\sigma_T\leq(\sqrt{1-\alpha^2})^{L(\alpha,\epsilon)-i_t}\cdot|u_{i_t}|/\alpha\leq C^{-1}\cdot|u_{i_t}|.
\end{align*}
Therefore, by Lemma~\ref{lm:auxi3}, for at least $1-\epsilon$ fraction of $\xb$, it holds with probability at least $1-\epsilon$ of $\yb$ that
\begin{align*}
    |\langle\ub_H,\yb_H\rangle-\theta|=\Big|\sum_{i\in H}u_iy_i-\theta\Big|\geq\frac{|u_{i_t}|}{4}\geq\frac{C\cdot\|\ub_T\|_2}{4}\geq\frac{|\langle\ub_T,\yb_T\rangle|}{4}.
\end{align*}
Then, it follows that with probability at least $1-\epsilon$ of $\yb$ we have
\begin{align*}
    \sign\big(\langle\ub_{H},\yb_{H}\rangle+\langle\ub_{T},\yb_{T}\rangle-\theta\big)=\sign\big(\langle\ub_{H},\yb_{H}\rangle-\theta\big).
\end{align*}
\end{proof}

Now we are ready to construct our polynomial. We consider the two cases in Lemma~\ref{lm:approx} separately. 

\textbf{Case 1:} If the weight vector $\wb$ of the LTF falls into the second case of Lemma~\ref{lm:approx}, notice that $\sign(\langle\ub_{H},\yb_{H}\rangle-\theta)=\sign(\langle\wb_{H}\odot\xb_{H},\yb_{H}\rangle-\theta)$ can be represented as a polynomial $p_{\yb}(\xb)$ of degree at most $H=K$ since only $H$ coordinates of $\xb$ are relevant. In this case, we take our final polynomial as
\begin{align*}
    p_{\zb}(\xb)=\EE_{\yb\sim\cN_{1-\rho}(\zb)}[p_{\yb}(\xb)]=\EE_{\yb\sim\cN_{1-\rho}(\zb)}[\sign(\langle\wb_{H}\odot\xb_{H},\yb_{H}\rangle-\theta)].
\end{align*}
Let $\Delta(\xb)$ be defined as the error term $\EE_{\zb\sim\cN_{\sigma}}[|p_{\zb}(\xb)-T_{1-\rho}f_{\xb}(\zb)|]$. We have that for at least $1-\epsilon$ fraction of $\xb$ it holds that
\begin{align*}
    \Delta(\xb)&=\EE_{\zb\sim\cN_{\sigma}}\Big[\Big|\EE_{\yb\sim\cN_{1-\rho}(\zb)}[\sign(\langle\wb_{H}\odot\xb_{H},\yb_{H}\rangle-\theta)]-\EE_{\yb\sim\cN_{1-\rho}(\zb)}[\sign(\langle\wb\odot\xb,\yb\rangle-\theta)]\Big|\Big]\\
    &\leq\EE_{\zb\sim\cN_{\sigma}}\Big[\EE_{\yb\sim\cN_{1-\rho}(\zb)}\big[\big|\sign(\langle\wb_{H}\odot\xb_{H},\yb_{H}\rangle-\theta)-\sign(\langle\wb\odot\xb,\yb\rangle-\theta)\big|\big]\Big]\\
    &=2\EE_{\zb\sim\cN_{\sigma}}\Big[\PP_{\yb\sim\cN_{1-\rho}(\zb)}\big[\sign(\langle\wb_{H}\odot\xb_{H},\yb_{H}\rangle-\theta)\neq\sign(\langle\wb\odot\xb,\yb\rangle-\theta)\big]\Big]\\
    &\leq2\epsilon.
\end{align*}
It follows that the final $L_1$ approximation error
\begin{align*}
    &\EE_{\xb\sim \cD_{\xb}}\big[\EE_{\zb\sim\cN_{\sigma}}[|p_{\zb}(\xb)-T_{1-\rho}f_{\xb}(\zb)|]\big]\\
    &=\EE_{\xb\sim D_{\xb}}[\Delta(\xb)]\\
    &=\EE_{\xb\sim D_{\xb}}[\Delta(\xb)|\text{``good'' }\xb]\cdot\PP_{\xb\sim D_{\xb}}[\text{``good'' }\xb]+\EE_{\xb\sim D_{\xb}}[\Delta(\xb)|\text{``bad'' }\xb]\cdot\PP_{\xb\sim D_{\xb}}[\text{``bad'' }\xb]\\
    &\leq2\epsilon\cdot 1+2\cdot\epsilon=4\epsilon.
\end{align*}
Here ``good'' $\xb$ refers to the at least $1-\epsilon$ fraction of $\xb$ such that the approximation in \eqref{ineq:approx} holds.

\textbf{Case 2:} If the weight vector $\wb$ of the LTF falls into the first case of Lemma~\ref{lm:approx}, we consider the following approximation
\begin{align*}
    p_{\yb_{H}}(\xb)&=\sign(\langle\ub_{H},\yb_{H}\rangle-\theta)\cdot\ind\big(|\langle\ub_{H},\yb_{H}\rangle-\theta|>C\cdot\|\ub_{T}\|_2\big)\\
    &\qquad+\tilde{p}_{\yb_{H}}(\xb)\cdot\ind\big(|\langle\ub_{H},\yb_{H}\rangle-\theta|\leq C\cdot\|\ub_{T}\|_2\big),
\end{align*}
where $C=(1-2\rho\sigma)\cdot\lambda\log^{\frac{1}{1+\alpha}}(2/\epsilon)+\sqrt{2\log(2/\epsilon)}$ and $\tilde{p}_{\yb_{H}}(\xb)$ will be choosen later as \eqref{eq:p}. In this case, we take our final polynomial as
\begin{align*}
    p_{\zb}(\xb)=\EE_{\yb_{H}\sim\cN_{1-\rho}(\zb)|_H}[p_{\yb_H}(\xb)].
\end{align*}
Let $\Delta(\xb)$ be defined as the $L_1$ error term $\EE_{\zb\sim\cN_{\sigma}}[|p_{\zb}(\xb)-T_{1-\rho}f_{\xb}(\zb)|]$. For notation simplicity, we denote $\{\yb_{H}:|\langle\ub_{H},\yb_{H}\rangle-\theta|>C\cdot\|\ub_{T}\|_2\}$ as event $\cE$. Then, we have
\begin{align*}
    &\Delta(\xb)\\
    &=\EE_{\zb\sim\cN_{\sigma}}\big[\big|\EE_{\yb\sim\cN_{1-\rho}(\zb)}[p_{\yb}(\xb)]-\EE_{\yb\sim\cN_{1-\rho}(\zb)}[f_{\xb}(\yb)]\big|\big]\\
    &=\EE_{\zb\sim\cN_{\sigma}}\Big[\Big|\EE_{\yb\sim\cN_{1-\rho}(\zb)}\big[(p_{\yb}(\xb)-f_{\xb}(\yb))\cdot\ind[\cE]\big]+\EE_{\yb\sim\cN_{1-\rho}(\zb)}\big[(p_{\yb}(\xb)-f_{\xb}(\yb))\cdot\ind[\cE^{c}]\big]\Big|\Big]\\
    &\leq\EE_{\zb\sim\cN_{\sigma}}\Big[\Big|\EE_{\yb\sim\cN_{1-\rho}(\zb)}\big[(p_{\yb}(\xb)-f_{\xb}(\yb))\cdot\ind[\cE]\big]\Big|+\Big|\EE_{\yb\sim\cN_{1-\rho}(\zb)}\big[(p_{\yb}(\xb)-f_{\xb}(\yb))\cdot\ind[\cE^{c}]\big]\Big|\Big]\\
    &=\underbrace{\EE_{\zb\sim\cN_{\sigma}}\Big[\Big|\EE_{\yb\sim\cN_{1-\rho}(\zb)}\big[\big(\sign(\langle\ub,\yb\rangle-\theta)-\sign(\langle\ub_{H},\yb_{H}\rangle-\theta)\big)\cdot\ind[\cE]\big]\Big|\Big]}_{\Delta_1(\xb)}\\
    &\qquad+\underbrace{\EE_{\zb\sim\cN_{\sigma}}\Big[\Big|\EE_{\yb\sim\cN_{1-\rho}(\zb)}\big[\big(\sign(\langle\ub,\yb\rangle-\theta)-\tilde{p}_{\yb}(\xb)\big)\cdot\ind[\cE^{c}]\big]\Big|\Big]}_{\Delta_2(\xb)}
\end{align*}
Notice that by Lemma~\ref{lm:auxi3}, for at least $1-\epsilon$ fraction of $\xb$, $|\langle\ub_{T},\yb_{T}\rangle|\leq C\cdot\|\ub_{T}\|_2$ holds for at least $1-\epsilon$ fraction of $\yb$. For such $\xb$ and $\yb$, under event $\cE$, we have
\begin{align*}
    |\langle\ub_T,\yb_T\rangle|\leq C\cdot\|\ub_T\|_2<|\langle\ub_H,\yb_H\rangle-\theta|,
\end{align*}
then it follows that
\begin{align*}
    \sign(\langle\ub,\yb\rangle-\theta)\cdot\ind[\cE]&=\sign(\langle\ub_{H},\yb_{H}\rangle+\langle\ub_{T},\yb_{T}\rangle-\theta)\cdot\ind[\cE]=\sign(\langle\ub_{H},\yb_{H}\rangle-\theta)\cdot\ind[\cE].
\end{align*}
Then, for at least $1-\epsilon$ fraction of $\xb$ we have
\begin{align*}
    &\Big|\EE_{\yb\sim\cN_{1-\rho}(\zb)}\big[\big(\sign(\langle\ub,\yb\rangle-\theta)-\sign(\langle\ub_{H},\yb_{H}\rangle-\theta)\big)\cdot\ind[\cE]\big]\Big|\\
    &\leq\Big|\EE_{\yb\sim\cN_{1-\rho}(\zb)}\big[\big(\sign(\langle\ub,\yb\rangle-\theta)-\sign(\langle\ub_{H},\yb_{H}\rangle-\theta)\big)\cdot\ind[\cE]\big|\text{``good'' }\yb\big]\Big|\\
    &\qquad\qquad\cdot\PP_{\yb\sim\cN_{1-\rho}(\zb)}[\text{``good'' }\yb]\\
    &\qquad+\Big|\EE_{\yb\sim\cN_{1-\rho}(\zb)}\big[\big(\sign(\langle\ub,\yb\rangle-\theta)-\sign(\langle\ub_{H},\yb_{H}\rangle-\theta)\big)\cdot\ind[\cE]\big|\text{``bad'' }\yb\big]\Big|\\
    &\qquad\qquad\qquad\cdot\PP_{\yb\sim\cN_{1-\rho}(\zb)}[\text{``bad'' }\yb]\\
    &\leq0+2\cdot\epsilon=2\epsilon.
\end{align*}
Here ``good'' $\yb$ refers to the at least $1-\epsilon$ fraction of $\yb$ such that $|\langle\ub_{T},\yb_{T}\rangle|\leq C\cdot\|\ub_{T}\|_2$ holds. Therefore, we have
\begin{align*}
    \EE_{\xb\sim\cD_{\xb}}[\Delta_1(\xb)]&=\EE_{\xb\sim\cD_{\xb}}[\Delta_1(\xb)|\text{``good'' }\xb]\cdot\PP_{\xb\sim\cD_{\xb}}[\text{``good'' }\xb]\\
    &\qquad+\EE_{\xb\sim\cD_{\xb}}[\Delta_1(\xb)|\text{``bad'' }\xb]\cdot\PP_{\xb\sim\cD_{\xb}}[\text{``bad'' }\xb]\\
    &\leq\EE_{\xb\sim\cD_{\xb}}[\Delta_1(\xb)|\text{``good'' }\xb]+2\cdot \epsilon\\
    &\leq 2\epsilon+2\epsilon=4\epsilon.
\end{align*}
Here ``good'' $\xb$ refers to the at least $1-\epsilon$ fraction of $\xb$ such that $\PP_{\yb}[|\langle\ub_T,\yb_T\rangle|\leq C\cdot\|\ub_T\|_2]\geq1-\epsilon$ holds. 

Next we consider bounding $\Delta_2(\xb)$ by constructing proper low-degree polynomial $\tilde{p}_{\yb}(\xb)$ as follows. Recall that $\yb\sim\cN_{1-\rho}(\zb)$ where $y_i=z_i$ with probability $1-\rho$ and $y_i$ randomly drawn from $\cN_{\sigma}$ with probability $\rho$, we use the following rerandomization trick for random vector $\yb_{T}$: for each coordinate of $\yb_{T}$, let
\begin{equation}\label{eq:decomp}
\left.
\begin{array}{cc}
    &y_i=(1-l_i)z_i+l_i\tau_i\\
    &\tau_i=(1-m_i)+m_i\epsilon_i
\end{array}
\right\}\Longrightarrow y_i=(1-l_i)z_i+l_i(1-m_i)+l_im_i\epsilon_i,
\end{equation}
where 
\begin{equation*}
l_i=\begin{cases}
    1\text{ w.p. }\rho\\
    0\text{ w.p. }1-\rho
\end{cases},
\tau_i=\begin{cases}
    1\text{ w.p. }1-\sigma\\
    -1\text{ w.p. }\sigma
\end{cases},
m_i=\begin{cases}
    1\text{ w.p. }2\sigma\\
    0\text{ w.p. }1-2\sigma
\end{cases},
\end{equation*}
and $\epsilon_i$ is a Radmacher random variable. Let random variable $A=\langle\ub,\yb\rangle-\theta$. Then by \eqref{eq:decomp} we have
\begin{align*}
    A&=\langle\ub,\yb\rangle-\theta\\
    &=\langle\ub_H,\yb_H\rangle+\langle\ub_T,\yb_T\rangle-\theta\\
    &=\langle\ub_{H},\yb_{H}\rangle+\langle\ub_T,(\one_T-\lb_T)\odot\zb_{T}\rangle+\langle\ub_{T},\lb_T\odot(\one_T-\mb_T)\rangle-\theta+\langle\ub_{T},\lb_T\odot\mb_{T}\odot\bepsilon_{T}\rangle\\
    &=\underbrace{\langle\ub_{H},\yb_{H}\rangle+\langle\ub_T,\vb_T\rangle-\theta}_{:=b}+\underbrace{\langle\ub_{T},\lb_T\odot\mb_{T}\odot\bepsilon_{T}\rangle}_{:=B}.
\end{align*}
where $\vb_T:=(\one_T-\lb_T)\odot\zb_T+\lb_T\odot(\one_T-\mb_T)$. Then we have
\begin{align*}
    \tilde{T_{1-\rho}f_{\xb}(\zb)}&:=\EE_{\yb\sim\cN_{1-\rho}(\zb)}\big[\sign(\langle\ub,\yb\rangle-\theta)\cdot\ind[\yb_{H}\in\cE^{c}]\big]\\
    &=\EE_{\yb_H,A}[\sign(A)\cdot\ind[\yb_{H}\in\cE^{c}]]\\
    &=\EE_{\yb_{H},\lb_{T},\mb_{T}}\big[\EE_{A}[\sign(A)|\yb_{H},\lb_{T},\mb_{T}]\cdot\ind[\yb_{H}\in\cE^{c}]\big]\\
    &=\EE_{\yb_{H},\lb_{T},\mb_{T}}\big[\big(1-2\PP_{A}[A\leq 0|\yb_{H},\lb_{T},\mb_{T}]\big)\cdot\ind[\yb_{H}\in\cE^{c}]\big]\\
    &=\EE_{\yb_{H},\lb_{T},\mb_{T}}\big[\big(1-2\PP_{B}[B\leq -b|\yb_{H},\lb_{T},\mb_{T}]\big)\cdot\ind[\yb_{H}\in\cE^{c}]\big],
\end{align*}
and
\begin{align*}
    &\Delta_2(\xb)\\
    &=\EE_{\zb\sim\cN_{\sigma}}\Big[\Big|\EE_{\yb\sim\cN_{1-\rho}(\zb)}\big[\sign(\langle\ub,\yb\rangle-\theta)\cdot\one[\yb_{H}\in\cE^c]\big]-\EE_{\yb\sim\cN_{1-\rho}(\zb)}\big[\tilde{p}_{\yb}(\xb)\cdot\one[\yb_{H}\in\cE^c]\big]\Big|\Big]\\
    &=\EE_{\zb\sim\cN_{\sigma}}\Big[\Big|\tilde{T_{1-\rho}f_{\xb}(\zb)}-\EE_{\yb\sim\cN_{1-\rho}(\zb)}\big[\tilde{p}_{\yb}(\xb)\cdot\one[\yb_{H}\in\cE^c]\big]\Big|\Big].
\end{align*}
By Theorem~\ref{thm:Berry-Esseen}, we have
\begin{align*}
    \sup_{x\in\RR}\bigg|\PP\bigg[\frac{B}{\|\ub_{T}\odot\lb_{T}\odot\mb_{T}\|_2}\leq x\bigg|\lb_{T},\mb_{T}\bigg]-\Phi(x)\bigg|\leq C'\cdot\bigg(\frac{\|\ub_{T}\odot\lb_{T}\odot\mb_{T}\|_3}{\|\ub_{T}\odot\lb_{T}\odot\mb_{T}\|_2}\bigg)^3.
\end{align*}
where $C'$ is a constant. Therefore, we have
\begin{align*}
    &\Bigg|\tilde{T_{1-\rho}f_{\xb}(\zb)}-\EE_{\yb_{H},\lb_{T},\mb_{T}}\bigg[\bigg(1-2\Phi\bigg(\frac{-b}{\|\ub_{T}\odot\lb_{T}\odot\mb_{T}\|_{2}}\bigg)\bigg)\cdot\ind[\yb_H\in\cE^{c}]\bigg]\Bigg|\\
    &\leq2\EE_{\yb_H,\lb_T,\mb_T}\bigg[\bigg|\PP[B\leq -b|\yb_H,\lb_T,\mb_T]-\Phi\bigg(\frac{-b}{\|\ub_{T}\odot\lb_{T}\odot\mb_{T}\|_{2}}\bigg)\bigg|\cdot\ind[\yb_H\in\cE^{c}]\bigg]\\
    &\leq 2C'\cdot\EE_{\yb_H,\lb_T,\mb_T}\bigg[\bigg(\frac{\|\ub_{T}\odot\lb_{T}\odot\mb_{T}\|_3}{\|\ub_{T}\odot\lb_{T}\odot\mb_{T}\|_2}\bigg)^3\cdot\ind[\yb_H\in\cE^{c}]\bigg]\\
    &\leq2C'\cdot\EE_{\yb_H,\lb_T,\mb_T}\Bigg[\frac{\|\ub_{T}\odot\lb_{T}\odot\mb_{T}\|_{\infty}}{\|\ub_{T}\odot\lb_{T}\odot\mb_{T}\|_2}\cdot\ind[\yb_H\in\cE^{c}]\Bigg].
\end{align*}
Notice that in this case $\ub_{T}$ is $\alpha$-regular and $l_i m_i$ takes $1$ with probability $2\rho\sigma$ and $0$ with probability $1-2\rho\sigma$, then by Lemma~\ref{lm:regularity}, for at least $1-\epsilon$ fraction of $(\lb_{T},\mb_{T})$ it holds that
\begin{align}
    \|\ub_{T}\|_{\infty}\leq\alpha\cdot\|\ub_{T}\|_{2}\leq(\rho\sigma)^{-\frac{1}{2}}\cdot\alpha\cdot\|\ub_{T}\odot\lb_{T}\odot\mb_{T}\|_{2},\label{ineq:bound2}
\end{align}
as long as the condition $\alpha\leq\rho\sigma/\sqrt{\log(1/\epsilon)/2}$ holds. 

Notice that
\begin{align*}
    1-2\Phi\bigg(\frac{-b}{\|\ub_{T}\odot\lb_{T}\odot\mb_{T}\|_{2}}\bigg)&=\EE_{X\sim\cN(b,\|\ub_{T}\odot\lb_{T}\odot\mb_{T}\|_2^2)}[\sign(X)|\yb_{H},\lb_{T},\mb_{T}]\\
    &=\EE_{X\sim\cN(b,\|\wb_{T}\odot\lb_{T}\odot\mb_{T}\|_2^2)}[\sign(X)|\yb_{H},\lb_{T},\mb_{T}],
\end{align*}
let
\begin{align*}
    \overline{T_{\rho}f_{\xb}(\zb)}&:=\EE_{\yb_{H},\lb_{T},\mb_{T}}\Big[\EE_{X\sim\cN(b,\|\wb_{T}\odot\lb_{T}\odot\mb_{T}\|_2^2)}[\sign(X)|\yb_{H},\lb_{T},\mb_{T}]\\
    &\qquad\cdot\ind[(\lb_T,\mb_T)\in\cE_{1}]\cdot\ind[\yb_{H}\in\cE^{c}]\Big],
\end{align*}
where $\cE_{1}$ is the event regarding the randomness of $(\lb_{T},\mb_{T})$ such that \eqref{ineq:bound2} holds, then
\begin{align*}
    &\big|\tilde{T_{\rho}f_{\xb}(\zb)}-\overline{T_{\rho}f_{\xb}(\zb)}\big|\\
    &\leq2\EE_{\yb_{H},\lb_{T},\mb_{T}}\bigg[\Big|\PP[B\leq-b|\yb_H,\lb_T,\mb_T]-\Phi\Big(-\frac{b}{\|\ub_{T}\odot\lb_{T}\odot\mb_{T}\|_{2}}\Big)\Big|\\
    &\qquad\qquad\qquad\qquad\cdot\ind[(\lb_T,\mb_T)\in\cE_{1}]\cdot\ind[\yb_{H}\in\cE^{c}]\bigg]\\
    &\qquad+\EE_{\yb_{H},\lb_{T},\mb_{T}}\Big[\big|1-2\PP[B\leq -b|\yb_{H},\lb_{T},\mb_{T}]\big|\cdot\ind[(\lb_T,\mb_T)\in\cE_{1}^{c}]\cdot\ind[\yb_{H}\in\cE^{c}]\Big]\\
    &\leq 2C'\cdot\EE_{\yb_H,\lb_T,\mb_T}\bigg[\frac{\|\ub_{T}\odot\lb_{T}\odot\mb_{T}\|_{\infty}}{\|\ub_{T}\odot\lb_{T}\odot\mb_{T}\|_2}\cdot\ind[(\lb_T,\mb_T)\in\cE_{1}]\cdot\ind[\yb_{H}\in\cE^{c}]\bigg]\\
    &\qquad+\EE_{\yb_{H},\lb_{T},\mb_{T}}\Big[\ind[(\lb_T,\mb_T)\in\cE_{1}^{c}]\cdot\ind[\yb_{H}\in\cE^{c}]\Big]\\
    &\leq 2C'\cdot(\rho\sigma)^{-\frac{1}{2}}\alpha+\epsilon\\
    &\leq 2C'\cdot(2\rho\sigma/\log(1/\epsilon))^{\frac{1}{2}}+\epsilon\\
    &\leq2\epsilon,
\end{align*}
as long as $\rho=O(\epsilon^2\log(1/\epsilon)/\sigma)$. Then, we have
\begin{align*}
    \EE_{\xb\sim \cD_{\xb}}\EE_{\zb\sim\cN_{\sigma}}\big[\big|\tilde{T_{\rho}f_{\xb}(\zb)}-\overline{T_{\rho}f_{\xb}(\zb)}\big|\big]\leq2\epsilon.
\end{align*}

Now we only need to consider polynomial approximation for $\overline{T_{\rho}f_{\xb}(\zb)}$. We can recenter the expectation around zero as follows:
\begin{align*}
    &\EE_{X\sim\cN(b,\|\wb_{T}\odot\lb_{T}\odot\mb_{T}\|_2^2)}[\sign(X)|\yb_{H},\lb_{T},\mb_{T}]\\
    &=\EE_{s\sim\cN(b/\|\wb_{T}\odot\lb_{T}\odot\mb_{T}\|_2,1)}[\sign(\|\wb_{T}\odot\lb_{T}\odot\ub_{T}\|_{2}\cdot s)|\yb_{H},\lb_{T},\mb_{T}]\\
    &=\EE_{s\sim Q}\bigg[\sign(\|\wb_{T}\odot\lb_{T}\odot\mb_{T}\|_2\cdot s)\cdot\frac{\cN(s;b/\|\wb_{T}\odot\lb_{T}\odot\mb_{T}\|_2,1)}{Q(s)}\bigg|\yb_{H},\lb_{T},\mb_{T}\bigg]\\
    &=\frac{e^{-\frac{b^2}{2\|\wb_{T}\odot\lb_{T}\odot\mb_{T}\|_2^2}}}{\sqrt{2\pi}}\cdot\EE_{s\sim Q}\bigg[\sign(\|\wb_{T}\odot\lb_{T}\odot\mb_{T}\|_{2}\cdot s)\cdot e^{-\frac{s^2}{2}-\log Q(s)}\\
    &\qquad\qquad\qquad\qquad\qquad\qquad\qquad\cdot e^{\frac{b\cdot s}{\|\wb_{T}\odot\lb_{T}\odot\mb_{T}\|_2}}\bigg|\yb_{H},\lb_{T},\mb_{T}\bigg],
\end{align*}
where $Q(s)=e^{-|s|}/2$.

For the simplicity of the following analysis, we define random variable
\begin{align*}
    x&:=\frac{\langle\ub_{T},\vb_T\rangle}{\|\ub_T\odot\vb_T\|_2}=\frac{\langle\wb_{T}\odot\xb_{T},\vb_T\rangle}{\|\ub_T\odot\vb_T\|_2}=\frac{\langle\wb_{T}\odot\vb_T,\xb_{T}\rangle}{\|\wb_T\odot\vb_T\|_2}.
\end{align*}
Since $\xb$ is $(\alpha,\lambda)$-strictly sub-exponential, then $x$ satisfies the following concentration inequality:
\begin{align*}
    \PP_{\xb\sim\cD_{\xb}}\big[|x|>t\big]\leq2\cdot e^{-(t/\lambda)^{1+\alpha}}.
\end{align*}
Then, we can rewrite:
\begin{align*}
    \frac{b}{\|\wb_T\odot\lb_T\odot\mb_T\|_2}&=\frac{(\langle\ub_H,\yb_H\rangle-\theta)+\langle\ub_{T},\vb_T\rangle}{\|\wb_T\odot\lb_T\odot\mb_T\|_2}=\frac{(\langle\ub_H,\yb_H\rangle-\theta)+\|\wb_{T}\odot\vb_T\|_2\cdot x}{\|\wb_T\odot\lb_T\odot\mb_T\|_2}.
\end{align*}
For simplicity, denote
\begin{align*}
    a=\frac{\|\wb_{T}\odot\vb_T\|_2}{\|\wb_T\odot\lb_T\odot\mb_T\|_2},\qquad c=\frac{\langle\ub_H,\yb_H\rangle-\theta}{\|\wb_T\odot\lb_T\odot\mb_T\|_2}.
\end{align*}
Notice that under condition $(\lb_T,\mb_{T})\in\cE_1$ and condition $\yb_{H}\in\cE^{c}$, we have
\begin{align}
    |c|&=\frac{|\langle\ub_H,\yb_H\rangle-\theta|}{\|\wb_T\odot\lb_T\odot\mb_T\|_2}\leq C\cdot\frac{\|\ub_T\|_2}{\|\wb_T\odot\lb_T\odot\mb_T\|_2}\leq C\cdot(\rho\sigma)^{-1/2}:=c',\label{ineq:c}\\
    |a|&\leq\frac{\|\wb_T\odot(\one_T-\lb_T)\odot\zb_T\|_2+\|\wb_T\odot\lb_T\odot(\one_T-\mb_T)\|_2}{\|\wb_T\odot\lb_T\odot\mb_T\|_2}\notag\\
    &\leq\frac{\|\wb_T\odot(\one_T-\lb_T)\|_2+\|\wb_T\odot\lb_T\|_2}{\|\wb_T\odot\lb_T\odot\mb_T\|_2}\notag\\
    &=\frac{\|\wb_T\|_2}{\|\wb_T\odot\lb_T\odot\mb_T\|_2}\notag\\
    &\leq (\rho\sigma)^{-1/2}:=a',\label{ineq:a}
\end{align}
where $C=(1-2\rho\sigma)\cdot\lambda\log^{\frac{1}{1+\alpha}}(4/\epsilon)+\sqrt{2\log(2/\epsilon)}$. We also have
\begin{align*}
    \frac{b}{\|\wb_T\odot\lb_T\odot\mb_T\|_2}=ax+c,
\end{align*}
and hence
\begin{align*}
    &\EE_{X\sim\cN(b,\|\wb_T\odot\lb_T\odot\mb_{T}\|_2^2)}[\sign(X)|\yb_{H},\lb_T,\mb_{T}]\\
    &=\frac{e^{-\frac{(ax+c)^2}{2}}}{\sqrt{2\pi}}\cdot\EE_{s\sim Q}\Big[\sign(\|\wb_T\odot\lb_T\odot\mb_{T}\|_{2}\cdot s)\cdot e^{-\frac{s^2}{2}-\log Q(s)}\cdot e^{(ax+c)s}\Big|\yb_{H},\lb_T,\mb_{T}\Big]\\
    &=\frac{e^{-\frac{a^2x^2}{2}}}{\sqrt{2\pi}}\cdot\EE_{s\sim Q}\Big[\sign(\|\wb_T\odot\lb_T\odot\mb_{T}\|_{2}\cdot s)\cdot e^{-\frac{s^2}{2}-\log Q(s)}\cdot e^{a(s-c)x+c(s-\frac{1}{2}c)}\Big|\yb_{H},\lb_T,\mb_{T}\Big].
\end{align*}
Then, we have
\begin{align*}
    &\overline{T_{\rho}f_{\xb}(\zb)}\\
    &=\EE_{\yb_{H},\lb_{T},\mb_{T}}\Big[\EE_{X\sim\cN(b,\|\wb_{T}\odot\lb_{T}\odot\mb_{T}\|_2^2)}[\sign(X)|\yb_{H},\lb_{T},\mb_{T}]\cdot\ind[(\lb_T,\mb_T)\in\cE_{1}]\cdot\ind[\yb_{H}\in\cE^{c}]\Big]\\
    &=\EE_{\yb_{H},\lb_{T},\mb_{T}}\bigg[e^{-\frac{a^2x^2}{2}}\cdot\ind[(\lb_T,\mb_T)\in\cE_{1}]\cdot\ind[\yb_{H}\in\cE^{c}]\\
    &\qquad\cdot\EE_{s\sim Q}\Big[\sign(\|\wb_T\odot\lb_T\odot\mb_{T}\|_{2}\cdot s)\cdot e^{-\frac{s^2}{2}-\log Q(s)}\cdot e^{a(s-c)x+c(s-\frac{1}{2}c)}\Big|\yb_{H},\lb_T,\mb_{T}\Big]\bigg].
\end{align*}

We now define a polynomial $\tilde{p}_{\zb}(\xb)$ approximating $\overline{T_{\rho}f_{\xb}(\zb)}$. To do this, we approximate $e^{-\frac{1}{2}a^2x^2}$ and $e^{a(s-c)x}$ using polynomials in $x$. First, we use a polynomial $p_1(x)$ to approximate $e^{-\frac{1}{2}a^2x^2}$. This polynomial is given by the following lemma. We choose the parameters later.

\begin{lemma}\label{lm:approx2}
Let $t\in\ZZ_{+}$. Let $x$ be a random variable satisfying the $(\alpha,\lambda)$-strictly sub-exponential tail bound. Then there exists a polynomial $q$ of degree
\begin{align*}
    O\Big((a^2\lambda^2/2)^{1+1/\alpha}\big(Cb\log(ab\lambda\sqrt{\log(1/\epsilon)})\big)^{2/\alpha}\big(\log(1/\epsilon)\big)^{\max\{1,1/2+1/\alpha\}}C^{1/\alpha^2}\Big)
\end{align*}
where $C$ is a sufficiently large constant such that the approximation error $\EE_{x}[(q(x)-e^{-\frac{1}{2}a^2x^2})^b]$ is upper bounded by $2\epsilon$.

\end{lemma}

Second, to approximate $e^{a(s-c)x}$, we use the function $p_2(x,s)=p_{k}(a(s-c)x)\ind[|s|\leq T]$ where $p_k(x)=1+\sum_{i=1}^{k-1}\frac{x^i}{i!}$ is the degree $k-1$ Taylor approximation of $e^x$. We choose degree $k$ and threshold $T$ later. Thus our final approximation of $\overline{T_{\rho}f_{\xb}(\zb)}$ is

\begin{align}
    \tilde{p}_{\xb}(\yb_{H},\lb_{T},\mb_{T})&=p_1(x)\cdot\EE_{s\sim Q}\big[\sign(\|\wb_{T}\odot\lb_{T}\odot\mb_T\|_{2}\cdot s)\cdot e^{-\frac{s^2}{2}-\log Q(s)}\notag\\
    &\qquad\cdot e^{c(s-\frac{1}{2}c)}\cdot p_2(x,s)\big|\yb_{H},\lb_T,\mb_{T}\big],\notag\\
    \tilde{p}_{\yb_{H}}(\xb)&=\EE_{\lb_{T},\mb_{T}}\big[\tilde{p}_{\xb}(\yb_{H},\lb_{T},\mb_{T})\cdot\ind[(\lb_T,\mb_T)\in\cE_{1}]\big],\label{eq:p}\\
    \tilde{p}_{\zb}(\xb)&=\EE_{\yb_{H}}\big[\tilde{p}_{\yb_{H}}(\xb)\cdot\ind[\yb_{H}\in\cE^{c}]\big].\notag
\end{align}

We now want to bound the $L_1$ error term $\EE_{\xb\sim D_{\xb}}\EE_{\zb\sim\cN_{\sigma}}[|\tilde{p}_{\zb}(\xb)-\overline{T_{\rho}f_{\xb}(\zb)}|]$. To help us analyse the error, we define the ``hybrid'' function $\bar{p}_{\zb}(\xb)$ such that 
\begin{align*}
    \bar{p}_{\xb}(\yb_H,\lb_T,\mb_T)&=\frac{e^{-\frac{a^2x^2}{2}}}{\sqrt{2\pi}}\cdot\EE_{s\sim\cN(0,1)}\big[\sign(\|\wb_{T}\odot\lb_{T}\odot\mb_T\|_{2}\cdot s)\cdot e^{-\frac{s^2}{2}-\log Q(s)}\\
    &\qquad\cdot e^{c(s-\frac{1}{2}c)}\cdot p_2(x,s)\big|\yb_{H},\lb_T,\mb_{T}\big],\\
    \bar{p}_{\zb}(\xb)&=\EE_{\yb_{H},\lb_T,\mb_{T}}\big[\bar{p}_{\xb}(\yb_{H},\lb_T,\mb_{T})\cdot\ind[(\lb_T,\mb_T)\in\cE_{1}]\cdot\ind[\yb_{H}\in\cE^{c}]\big].
\end{align*}

We have that
\begin{align*}
    &\EE_{\xb\sim \cD_{\xb}}\EE_{\zb\sim\cN_{\sigma}}[|\tilde{p}_{\zb}(\xb)-\overline{T_{1-\rho}f_{\xb}(\zb)}|]\\
    &\leq2\cdot\EE_{\xb\sim \cD_{\xb}}\Big[\underbrace{\EE_{\zb\sim\cN_{\sigma}}[|\bar{p}_{\zb}(\xb)-\overline{T_{1-\rho}f_{\xb}(\zb)}|]}_{\Delta_3(\xb)}+\underbrace{\EE_{\zb\sim\cN_{\sigma}}[|\tilde{p}_{\zb}(\xb)-\bar{p}_{\zb}(\xb)|]}_{\Delta_4(\xb)}\Big].
\end{align*}
We now bound $\Delta_3(\xb)$ and $\Delta_4(\xb)$ separately. We have that
\begin{align*}
    \Delta_3(\xb)&\leq\EE_{\zb}\bigg[\EE_{\yb_{H},\lb_T,\mb_{T}}\Big[e^{-\frac{1}{2}a^2x^2}\cdot\EE_{s\sim\cN(0,1)}\big[e^{-\frac{s^2}{2}-\log Q(s)}\cdot e^{c(s-\frac{1}{2}c)}\cdot|e^{a(s-c)x}-p_2(x,s)|\big]\\
    &\qquad\cdot\ind[(\lb_T,\mb_T)\in\cE_{1}]\cdot\ind[\yb_{H}\in\cE^{c}]\Big]\bigg].
\end{align*}
Observe that $\Delta_3(\xb)$ can be bounded as the expected sum of the following two terms:
\begin{align*}
    &\Delta_{31}(\xb,\yb_H,\lb_T,\mb_T)\\
    &=\frac{e^{-\frac{a^2x^2}{2}}}{\sqrt{2\pi}}\cdot\EE_{s\sim Q}\bigg[e^{-\frac{s^2}{2}-\log Q(s)}\cdot e^{c(s-\frac{1}{2}c)}\cdot\frac{e^{|a(s-c)x|}}{k!}\cdot|a(s-c)x|^{k}\cdot\ind[|s|\leq T]\bigg],\\
    &\Delta_{32}(\xb,\yb_H,\lb_T,\mb_T)\\
    &=\frac{e^{-\frac{a^2x^2}{2}}}{\sqrt{2\pi}}\cdot\EE_{s\sim Q}\big[e^{-\frac{s^2}{2}-\log Q(s)}\cdot e^{c(s-\frac{1}{2}c)}\cdot e^{a(s-c)x}\cdot\ind[|s|>T]\big],
\end{align*}
where the first term's bound comes from the fact that $|p_k(x)-e^x|\leq\frac{e^{|x|}}{k!}\cdot|x|^k$. 

We first bound $\Delta_{31}$. We have that
\begin{align*}
    &\Delta_{31}(\xb,\yb_H,\lb_T,\mb_T)\\
    &\leq \frac{e^{-\frac{a^2x^2}{2}}}{\sqrt{2\pi}}\cdot\EE_{s\sim Q}\bigg[e^{-\frac{s^2}{2}-\log Q(s)}\cdot e^{c(s-\frac{1}{2}c)}\cdot e^{|a(s-c)x|}\cdot\ind[|s|\leq T]\bigg]\cdot\frac{(|acx|+|aTx|)^{k}}{k!}\\
    &\leq \frac{(|acx|+|aTx|)^{k}}{k!}\cdot\EE_{s\sim Q}\bigg[\frac{e^{-\frac{1}{2}(s-(ax+c))^2}}{\sqrt{2\pi}\cdot Q(s)}\cdot\ind[|s|\leq T]\bigg]\\
    &\qquad+\frac{(|acx|+|aTx|)^{k}}{k!}\cdot\EE_{s\sim Q}\bigg[\frac{e^{-\frac{1}{2}(s-(c-ax))^2}}{\sqrt{2\pi}\cdot Q(s)}\cdot\ind[|s|\leq T]\bigg]\\
    &\leq\frac{(|acx|+|aTx|)^{k}}{(k!)^2}\cdot\EE_{s\sim Q}\bigg[\frac{\cN(s;ax+c,1)}{Q(s)}\bigg]\\
    &\qquad+\frac{(|acx|+|aTx|)^{2k}}{(k!)^2}\cdot\EE_{s\sim Q}\bigg[\frac{\cN(s;c-ax,1)}{Q(s)}\bigg]\\
    &=\frac{2(|acx|+|aTx|)^{k}}{k!}
\end{align*}
where the second inequality is by $e^{|a(s-c)x|}\leq e^{a(s-c)x}+e^{-a(s-c)x}$. 

Then, we have
\begin{align*}
    &\EE_{\xb\sim\cD_{\xb}}\bigg[\EE_{\zb\sim\cN_{\sigma}}\Big[\EE_{\yb_H,\lb_T,\mb_T}\big[\Delta_{31}(\xb,\yb_H,\lb_T,\mb_T)\cdot\ind[(\lb_T,\mb_T)\in\cE_{1}]\cdot\ind[\yb_{H}\in\cE^{c}]\big]\Big]\bigg]\\
    &\leq2\EE_{\xb\sim \cD_{\xb}}\Bigg[\EE_{\zb\sim\cN_{\sigma}}\bigg[\EE_{\yb_H,\lb_T,\mb_T}\Big[\frac{(|acx|+|aTx|)^{k}}{k!}\cdot\ind[(\lb_T,\mb_T)\in\cE_{1}]\cdot\ind[\yb_{H}\in\cE^{c}]\Big]\bigg]\Bigg]\\
    &=2\EE_{\zb\sim\cN_{\sigma}}\Bigg[\EE_{\yb_H,\mb_T}\bigg[\EE_{\xb\sim \cD_{\xb}}\Big[\frac{(|acx|+|aTx|)^{k}}{k!}\cdot\ind[(\lb_T,\mb_T)\in\cE_{1}]\cdot\ind[\yb_{H}\in\cE^{c}]\Big]\bigg]\Bigg].
\end{align*}
Under condition $(\lb_T,\mb_T)\in\cE_1$ and condition $\yb_H\in\cE^c$, we have \eqref{ineq:c} and \eqref{ineq:a} and hence
\begin{align*}
    &\EE_{\xb\sim \cD_{\xb}}\Big[\frac{(|acx|+|aTx|)^{k}}{k!}\cdot\ind[(\lb_T,\mb_T)\in\cE_{1}]\cdot\ind[\yb_{H}\in\cE^{c}]\Big]\\
    &\leq\frac{|a'|^{k}(|c'|+|T|)^{k}}{k!}\cdot\EE_{\xb\sim D}[|x|^k]\\
    &\leq\frac{|a'|^{k}(|c'|+|T|)^{k}}{k!}\cdot C\lambda^k\cdot \Big(\frac{k}{1+\alpha}\Big)^{\frac{k}{1+\alpha}+\frac{1}{2}}e^{-\frac{k}{1+\alpha}+\frac{1+\alpha}{12k}}\\
    &\leq\epsilon,
\end{align*}
when 
\begin{align}
    k\geq\big(C'|a'|(|c'|+|T|)\lambda\log(1/\epsilon)\big)^{1+\frac{1}{\alpha}}+(1+\alpha)\label{ineq:k}
\end{align}
and $C'$ is a large enough constant. The second inequality used Lemma~\ref{lm:auxi}.

We now bound $\Delta_{32}$. We have that
\begin{align*}
    \Delta_{32}(\xb,\yb_H,\lb_T,\mb_T)&=\EE_{s\sim Q}\bigg[\frac{e^{-\frac{1}{2}(s-(ax+c))^2}}{Q(s)}\cdot\ind[|s|> T]\bigg]\\
    &\leq\sqrt{\EE_{s\sim Q}\bigg[\bigg(\frac{\cN(s;ax+c,1)}{Q(s)}\bigg)^2\bigg]\cdot\PP_{s\sim Q}\big[|s|>T\big]}\\
    &\leq\sqrt{Ce^{|ax+c|}\cdot e^{-T}}\\
    &\leq C' e^{\frac{1}{2}(|ax|+|c|)}\cdot e^{-T/2}.
\end{align*}
The third inequality is based on the following claim.

\begin{lemma}\label{lm:tilting2}
Define the distribution $Q$ on $\RR$ with density function $Q(s)=e^{-|s|}/2$. Then there exist a universal constant $C$ such that for every $b\in\RR$, it holds that
\begin{align*}
    \EE_{s\sim Q}\bigg[\Big(\frac{\cN(s;b,1)}{Q(s)}\Big)^{2}\bigg]\leq C e^{|b|}.
\end{align*}
\end{lemma}

Thus, we have that
\begin{align*}
    &\EE_{\xb\sim \cD_{\xb}}\bigg[\EE_{\zb\sim\cN_{\sigma}}\Big[\EE_{\yb_H,\lb_T,\mb_T}\big[\Delta_{32}(\xb,\yb_H,\lb_T,\mb_T)\cdot\ind[(\lb_T,\mb_T)\in\cE_{1}]\cdot\ind[\yb_{H}\in\cE^{c}]\big]\Big]\bigg]\\
    &\leq\EE_{\xb\sim \cD_{\xb}}\bigg[\EE_{\zb\sim\cN_{\sigma}}\Big[\EE_{\yb_H,\lb_T,\mb_T}\big[C' e^{\frac{1}{2}(|ax|+|c|)}\cdot e^{-T/2}\cdot\ind[(\lb_T,\mb_T)\in\cE_{1}]\cdot\ind[\yb_{H}\in\cE^{c}]\big]\Big]\bigg]\\
    &=\EE_{\zb\sim\cN_{\sigma}}\bigg[\EE_{\yb_H,\lb_T,\mb_T}\Big[\EE_{\xb\sim \cD_{\xb}}\big[C' e^{\frac{1}{2}(|ax|+|c|)}\cdot e^{-T/2}\cdot\ind[(\lb_T,\mb_T)\in\cE_{1}]\cdot\ind[\yb_{H}\in\cE^{c}]\big]\Big]\bigg]\\
    &\leq\EE_{\zb\sim\cN_{\sigma}}\bigg[\EE_{\yb_H,\lb_T,\mb_T}\Big[\EE_{\xb\sim \cD_{\xb}}\big[C' e^{\frac{1}{2}|a'x|}\cdot\ind[(\lb_T,\mb_T)\in\cE_{1}]\cdot\ind[\yb_{H}\in\cE^{c}]\big]\cdot e^{\frac{|c'|-T}{2}}\Big]\bigg]\\
    &\leq\EE_{\zb\sim\cN_{\sigma}}\bigg[\EE_{\yb_H,\lb_T,\mb_T}\Big[\EE_{\xb\sim \cD_{\xb}}\big[C''e^{(a'\lambda)^{1+\frac{1}{\alpha}}/2}\cdot e^{\frac{|c'|-T}{2}}\big]\Big]\bigg]\\
    &=C''e^{(a'\lambda)^{1+\frac{1}{\alpha}}/2}\cdot e^{\frac{|c'|-T}{2}}\\
    &\leq\epsilon,
\end{align*}

when
\begin{align*}
    T\geq(a'\lambda)^{1+\frac{1}{\alpha}}+|c'|+2\log(C''/\epsilon).
\end{align*}
By \eqref{ineq:c} and \eqref{ineq:a}, we can take
\begin{align}
    T=O\Big((\rho\sigma)^{-\frac{1}{2}(1+\frac{1}{\alpha})}\lambda^{1+\frac{1}{\alpha}}\log(1/\epsilon)\Big).\label{ineq:T}
\end{align}

The third last inequality is by Lemma~\ref{lm:auxi1}. Plugging this into the bound for $k$ in \eqref{ineq:k}, we can take
\begin{align*}
    k\leq \Big(C'a'\lambda\big(2c'+(a'\lambda)^{1+\frac{1}{\alpha}}+2\log(C''/\epsilon)\big)\log(1/\epsilon)\Big)^{1+\frac{1}{\alpha}}+(1+\alpha)
\end{align*}
for constant $C',C''$. By \eqref{ineq:c} and \eqref{ineq:a}, we know that
\begin{align}
    k\leq \Big(C'''\lambda^{2+\frac{1}{\alpha}}(\rho\sigma)^{-(1+\frac{1}{2\alpha})}\log^3(1/\epsilon)\Big)^{1+\frac{1}{\alpha}}+(1+\alpha), \label{ineq:k2}%C'''(\rho\sigma)^{-2}\sigma_0^4\log^2(1/\epsilon),
\end{align}
where $C'''$ is a large constant.

We now bound $\Delta_4(\xb)$. We have that
\begin{align*}
    &\Delta_4(\xb)\\
    &\leq\EE_{\zb}\Big[\EE_{\yb_H,\lb_T,\mb_T}\big[\big|\tilde{p}_{\xb}(\yb_H,\lb_H,\mb_T)-\bar{p}_{\xb}(\yb_H,\lb_H,\mb_T)\big|\cdot\ind[(\lb_T,\mb_T)\in\cE_{1}]\cdot\ind[\yb_{H}\in\cE^{c}]\big]\Big]\\
    &=\frac{1}{\sqrt{2\pi}}\EE_{\zb}\Big[\EE_{\yb_H,\lb_T,\mb_T}\big[\big|e^{-\frac{1}{2}a^2x^2}-p_1(x)\big|\cdot\big|\EE_{s\sim Q}[e^{-\frac{s^2}{2}-\log Q(s)+c(s-\frac{c}{2})}\cdot p_2(x,s)]\big|\\
    &\qquad\qquad\cdot\ind[(\lb_T,\mb_T)\in\cE_{1}]\cdot\ind[\yb_{H}\in\cE^{c}]\big]\Big]
\end{align*}
Notice that
\begin{align*}
    &\big(\EE_{s\sim Q}[e^{-\frac{s^2}{2}-\log Q(s)+c(s-\frac{c}{2})}\cdot p_2(x,s)]\big)^2\\
    &\leq\EE_{s\sim Q}[e^{-s^2-2\log Q(s)+c(2s-c)}]\cdot \EE_{s\sim Q}[p_2(x,s)^2]\\
    &\leq \EE_{s\sim Q}\bigg[\bigg(\frac{\cN(s;c,1)}{Q(s)}\bigg)^2\bigg]\cdot\EE_{s\sim Q}\bigg[\bigg(\sum_{i=0}^{k-1}\frac{(a(s-c)x)^i}{i!}\bigg)^2\cdot\ind[|s|\leq T]\bigg]\\
    &\leq Ce^{|c|}\cdot \bigg(\sum_{i=0}^{k-1}\frac{|a|^i(|T|+|c|)^i|x|^i}{i!}\bigg)^2,
\end{align*}
where the last inequality is by Lemma~\ref{lm:tilting2}. Thus, we have
\begin{align*}
    &\EE_{\xb\sim\cD_{\xb}}[\Delta_4(\xb)]\\
    &\leq\EE_{\zb}\Big[\EE_{\yb_H,\lb_T,\mb_T}\big[\EE_{\xb\sim \cD_{\xb}}\big[\big|e^{-\frac{1}{2}a^2x^2}-p_1(x)\big|\cdot\big|\EE_{s\sim Q}[e^{-\frac{s^2}{2}-\log Q(s)+c(s-\frac{c}{2})}\cdot p_2(x,s)]\big|\big]\\
    &\qquad\qquad\cdot\ind[(\lb_T,\mb_T)\in\cE_{1}]\cdot\ind[\yb_{H}\in\cE^{c}]\big]\Big]\\
    &\leq\EE_{\zb}\Bigg[\EE_{\yb_H,\lb_T,\mb_T}\bigg[\EE_{\xb\sim \cD_{\xb}}\big[\big|e^{-\frac{1}{2}a^2x^2}-p_1(x)\big|\cdot C^{\frac{1}{2}}e^{\frac{|c|}{2}}\cdot \bigg(\sum_{i=0}^{k-1}\frac{|a|^i(|T|+|c|)^i|x|^i}{i!}\bigg)\\
    &\qquad\qquad\cdot\ind[(\lb_T,\mb_T)\in\cE_{1}]\cdot\ind[\yb_{H}\in\cE^{c}]\big]\bigg]\Bigg]
\end{align*}
Denote
\begin{align*}
    \Delta_5(\yb_H,\lb_T,\mb_T)&:=\EE_{\xb\sim \cD_{\xb}}\bigg[\big|e^{-\frac{1}{2}a^2x^2}-p_1(x)\big|\cdot C^{\frac{1}{2}}e^{\frac{|c|}{2}}\cdot \bigg(\sum_{i=0}^{k-1}\frac{|a|^i(|T|+|c|)^i|x|^i}{i!}\bigg)\\
    &\qquad\cdot\ind[(\lb_T,\mb_T)\in\cE_{1}]\cdot\ind[\yb_{H}\in\cE^{c}]\bigg].
\end{align*}
We have
\begin{align*}
    &\Delta_5(\yb_H,\lb_T,\mb_T)\\
    &\leq C'e^{|c'|/2}\sqrt{\EE_{\xb\sim \cD_{\xb}}\big[\big(e^{-\frac{1}{2}a^2x^2}-p_1(x)\big)^2\cdot\ind[(\lb_T,\mb_T)\in\cE_{1}]\cdot\ind[\yb_{H}\in\cE^{c}]\big]}\\
    &\qquad\cdot\sqrt{\EE_{\xb\sim \cD_{\xb}}\bigg[\bigg(\sum_{i=0}^{k-1}\frac{|a|^i(|T|+|c|)^i|x|^i}{i!}\bigg)^2\cdot\ind[(\lb_T,\mb_T)\in\cE_{1}]\cdot\ind[\yb_{H}\in\cE^{c}]\bigg]}\\
    &\leq C'e^{|c'|/2}\cdot\delta\cdot\sqrt{\EE_{\xb\sim \cD_{\xb}}\bigg[\bigg(\sum_{i=0}^{k-1}\frac{|a|^i(|T|+|c|)^i|x|^i}{i!}\bigg)^2\cdot\ind[(\lb_T,\mb_T)\in\cE_{1}]\cdot\ind[\yb_{H}\in\cE^{c}]\bigg]}\\
    &\leq C'e^{|c'|/2}\cdot\delta\cdot\sqrt{\bigg(\sum_{i=0}^{k-1}\frac{|a'|^i(|T|+|c'|)^i}{i!}\bigg)^2\cdot\max_{1\leq i\leq k-1}C\lambda^{2i}\cdot \Big(\frac{2i}{1+\alpha}\Big)^{\frac{2i}{1+\alpha}+\frac{1}{2}}e^{-\frac{2i}{1+\alpha}+\frac{1+\alpha}{24i}}}\\
    &\leq C''\delta e^{|c'|/2+|a'|(|T|+|c'|)}\lambda^{k}\Big(\frac{2k}{1+\alpha}\Big)^{\frac{k}{1+\alpha}+\frac{1}{4}}e^{\frac{1+\alpha}{48}}\\%\sigma_0^{2k}(4c_1 k)^{k}\\
    &\leq \epsilon,
\end{align*}
when $\delta$ is chosen accordingly. The third inequality is by Lemma~\ref{lm:auxi}.

By Lemma~\ref{lm:approx2} and taking the exponent $b$ as 2, the degree of $p_1(x)$ required to get the error is
\begin{align*}
    &O\Big((a^2\lambda^2/2)^{1+1/\alpha}\big(C\log(a\lambda\sqrt{\log(1/\delta)})\big)^{2/\alpha}\big(\log(1/\delta)\big)^{\max\{1,\frac{1}{2}+\frac{1}{\alpha}\}}C^{1/\alpha^2}\Big)\\
    &=O\Big((a'^2\lambda^2/2)^{1+1/\alpha}\big(C\log(a'\lambda)\big)^{2/\alpha}\\
    &\qquad\qquad\cdot\big(\log(1/\epsilon)+|a'|(|T|+|c'|)+(k+1/4)\log(2k\lambda)\big)^{2/\alpha+1}C^{1/\alpha^2}\Big)\\
    &=O\Big((a'^2\lambda^2/2)^{1+1/\alpha}\big(C\log(a'\lambda)\big)^{2/\alpha}\big(\log(1/\epsilon)+|a'|(|T|+|c'|)+k^{3/2}\lambda^{1/2}\big)^{2/\alpha+1}C^{1/\alpha^2}\Big).
\end{align*}
%$O((a')^2\sigma_0^{2.5}\log(a\sigma_0)(k\log(k\sigma_0^2)+|a'|(|T|+|c'|)+\log(1/\epsilon)))$.
By using \eqref{ineq:c} and \eqref{ineq:a} and plugging in the order of $k, T$ in \eqref{ineq:T} and \eqref{ineq:k2}, we know the degree of $p_1(x)$ is 
\begin{align*}
    &O\Big(((\rho\sigma)^{-1}\lambda^2/2)^{1+\frac{1}{\alpha}}\big(C\log((\rho\sigma)^{-\frac{1}{2}}\lambda)\big)^{\frac{2}{\alpha}}\\
    &\qquad\qquad\cdot\big(C\lambda^{\frac{3}{2}(2+\frac{1}{\alpha})(1+\frac{1}{\alpha})+\frac{1}{2}}(\rho\sigma)^{-\frac{3}{2}(1+\frac{1}{2\alpha})(1+\frac{1}{\alpha})}\log^{\frac{9}{2}(1+\frac{1}{\alpha})}(1/\epsilon)\big)^{1+\frac{2}{\alpha}}C^{\frac{1}{\alpha^2}}\Big)\\
    &=O\Big(\big(C(\rho\sigma)^{-\frac{1}{2}}\lambda\log(1/\epsilon)\big)^{6(1+\frac{1}{\alpha})^3}\Big).
\end{align*}

Putting everything together, we get that the degree of $p_{\zb}(\xb)$ is at most $2H+\deg(p_1)+\deg(p_2)$ which is
\begin{align*}
    O\big(\log(1+\lambda)/\alpha^2+\log(1/\epsilon)\log(1/\alpha)/\rho\sigma\alpha^2\big)+O\Big(\big(C(\rho\sigma)^{-\frac{1}{2}}\lambda\log(1/\epsilon)\big)^{6(1+\frac{1}{\alpha})^3}\Big).
\end{align*}
Plugging in the order of $\rho=\Omega(\epsilon^2)$ and $\alpha=\Omega(\rho\sigma/\sqrt{\log(1/\epsilon)})$, the degree can be bounded by 
\begin{align*}
    O\Big(\big(C\sigma^{-\frac{1}{2}}\lambda\log(1/\epsilon)/\epsilon\big)^{6(1+\frac{1}{\alpha})^3}\Big).
\end{align*}

\section{Proofs of Auxiliary Lemmas}

\subsection{Proof of Theorem~\ref{thm:auxi2}}
\begin{proof}
Let $f^*$ be the optimal halfspace that achieves $\mathrm{opt}_{\sigma}$. Let $p_{\zb}$ be the polynomial of degree at most $d$ such that
\begin{align}
    \EE_{\zb\sim\cD_{\sigma},\xb\sim\cD_{\xb}}[|p_{\zb}(\xb)-f^*(\xb\odot\zb)|]\leq\epsilon.\label{eq:1}
\end{align}
Consider the sample dataset $S=\{(\xb_i,y_i)\}_{i=1}^{N}$ in a single run of Algorithm~\ref{alg:l1-regression}. Let $p_S$ be the polynomial chosen by the algorithm and let $h_S$ be the corresponding hypothesis that the algorithm outputs. By the proof of Theorem~\ref{thm:auxi} in \cite{kalai2008agnostically}, we have that
\begin{align*}
    \frac{1}{N}\sum_{i=1}^{N}\ind[h_{S}(\xb_i)\neq y_i]\leq\frac{1}{2N}\sum_{i=1}^{N}|y_i-p_S(\xb_i)|.
\end{align*}
Notice that $p_S$ is the minimizer of the error, and thus beats any polynomial $p_{\zb}$ we choose, we have
\begin{align*}
    &\frac{1}{2N}\sum_{i=1}^{N}|y_i-p_S(\xb_i)|\leq\EE_{\zb\sim\cN_{\sigma}}\bigg[\frac{1}{2N}\sum_{i=1}^{N}|y_i-p_{\zb}(\xb_i)|\bigg]\\
    &\qquad\leq\underbrace{\EE_{\zb\sim\cN_{\sigma}}\bigg[\frac{1}{2N}\sum_{i=1}^{N}|f^*(\xb_i\odot\zb)-p_{\zb}(\xb_i)|\bigg]}_{\Delta_1(S)}+\underbrace{\EE_{\zb\sim\cN_{\sigma}}\bigg[\frac{1}{2N}\sum_{i=1}^{N}|y_i-f^*(\xb_i\odot\zb)|\bigg]}_{\Delta_2(S)}.
\end{align*}
By \eqref{eq:1}, we can bound $\Delta_1(S)$ as follows:
\begin{align*}
    \EE_{S\sim\cD^{\otimes n}}[\Delta_1(S)]&=\frac{1}{2}\EE_{\xb\sim\cD_{\xb},\zb\sim\cD_{\sigma}}[|f^*(\xb\odot\zb)-p_{\zb}(\xb)|]\leq\frac{1}{2}\epsilon.
\end{align*}
By the optimality of $f^*$, we have
\begin{align*}
    \EE_{S\sim\cD^{\otimes n}}[\Delta_1(S)]&=\frac{1}{2}\EE_{(\xb,y)\sim\cD,\zb\sim\cD_{\sigma}}[|y-f^*(\xb\odot\zb)|]\\
    &=\EE_{(\xb,y)\sim\cD,\zb\sim\cD_{\sigma}}[\ind[y\neq f^*(\xb\odot\zb)]]=\mathrm{opt}_{\sigma}.
\end{align*}
Thus, we obtain
\begin{align*}
    \EE_{S\sim\cD^{\otimes n}}\bigg[\frac{1}{N}\sum_{i=1}^{N}\ind[h_{S}(\xb_i)\neq y_i]\bigg]\leq\mathrm{opt}_{\sigma}+\frac{1}{2}\epsilon.
\end{align*}
Since our hypothesis $h_S$ is a polynomial threshold function of degree $d$ on $n$ variables, VC theory tells us that for $N=\poly(n^d/\epsilon)$, we have that
\begin{align*}
    \EE_{S\sim\cD^{\otimes n}}\big[\PP_{(\xb,y)\sim\cD}[h_{S}(\xb)\neq y]\big]\leq\mathrm{opt}_{\sigma}+\frac{3}{4}\epsilon.
\end{align*}
By Markov's inequality,  on any single repetition of the algorithm, we have that
\begin{align*}
    \PP_{S\sim\cD^{\otimes n}}\Big[\PP_{(\xb,y)\sim\cD}[h_{S}(\xb)\neq y]\geq\mathrm{opt}_{\sigma}+\frac{7}{8}\epsilon\Big]\leq\frac{\mathrm{opt}_{\sigma}+\frac{3}{4}\epsilon}{\mathrm{opt}_{\sigma}+\frac{7}{8}}\leq1-\frac{\epsilon}{16}.
\end{align*}
Hence, after $r=O(\log(1/\delta)/\epsilon)$  repetitions of the algorithm, with probability at least $1-\delta/2$, one of them
will have $\PP_{(\xb,y)\sim\cD}[h_{S}(\xb)\neq y]\leq\mathrm{opt}_{\sigma}+\frac{7}{8}\epsilon$. In this case, using an independent set of size $O(\log(1/\delta)/\epsilon^2)$, we probability at most $\delta/2$, we will choose one with error $>\mathrm{opt}_{\sigma}+\epsilon$.
\end{proof}

\subsection{Proof of Lemma~\ref{lm:tilting2}}
\begin{proof}
The proof below is straightforward calculation by completing the squares.
\begin{align*}
    \EE_{s\sim Q}\bigg[\Big(\frac{\cN(s;b,1)}{Q(s)}\Big)^{2}\bigg]&=\frac{1}{2}\int_{\RR}\frac{\frac{1}{2\pi}e^{-(s-b)^2}}{\frac{1}{4}e^{-2|s|}}\cdot e^{-|s|}ds\\
    &=\frac{1}{\pi}\int_{\RR}e^{-(s-b)^2+|s|}ds\\
    &\leq\frac{1}{\pi}\int_{\RR}e^{-(s-b)^2+s}ds+\frac{1}{\pi}\int_{\RR}e^{-(s-b)^2-s}ds\\
    &=\frac{1}{\pi}\int_{\RR}e^{-(s-b+\frac{1}{2})^2+b+\frac{1}{4}}ds+\frac{1}{\pi}\int_{\RR}e^{-(s-b-\frac{1}{2})^2-b+\frac{1}{4}}ds\\
    &=\frac{1}{\sqrt{\pi}}e^{b+\frac{1}{4}}+\frac{1}{\sqrt{\pi}}e^{-b+\frac{1}{4}}\\
    &\leq\frac{2e^{\frac{1}{4}}}{\sqrt{\pi}} e^{|b|}
\end{align*}
\end{proof}

\subsection{Proof of Lemma~\ref{lm:approx2}}
\begin{proof}
Let $p(x)=\sum_{i=0}^{\deg(p)}c_i x^i$ be the polynomial obtained from Lemma~\ref{lm:auxi2} with error $\epsilon/2$ and $T=\omega(\log(1/\epsilon))$ to be choosen later. Our final polynomial is $q(x)=p(\frac{1}{2}a^2x^2)$. Clearly, $\deg(q)=2\cdot\deg(p)=O(\sqrt{T\log(1/\epsilon)})$. We now bound the error.
\begin{align*}
    &\EE_{x}[(q(x)-e^{-\frac{1}{2}a^2x^2})^b]\\
    &=\EE_{x}\big[(q(x)-e^{-\frac{1}{2}a^2x^2})^b\cdot\ind[a^2x^2/2<T]\big]+\EE_{x}\big[(q(x)-e^{-\frac{1}{2}a^2x^2})^b\cdot\ind[a^2x^2/2\geq T]\big]\\
    &=\epsilon+\sqrt{\EE_{x}\big[(q(x)-e^{-\frac{1}{2}a^2x^2})^{2b}\big]\cdot\EE_{x}\big[\ind[a^2x^2/2\geq T]\big]}\\
    &\leq\epsilon+\sqrt{\EE_{x}\big[(q(x)-e^{-\frac{1}{2}a^2x^2})^{2b}\big]\cdot 2e^{-(\frac{2T}{a^2\lambda^2})^{\frac{1+\alpha}{2}}}},
\end{align*}
where the last inequality is by the tail bound of $(\alpha,\lambda)$-strictly sub-exponential random variable. 

We now bound $\EE_{x}\big[(q(x)-e^{-\frac{1}{2}a^2x^2})^{2b}\big]$. We have that
\begin{align*}
    \EE_{x}\big[(q(x)-e^{-\frac{1}{2}a^2x^2})^{2b}\big]&\leq\EE_{x}\big[(|q(x)|+1)^{2b}\big]\\
    &\leq\EE_{x}\bigg[\bigg(1+\sum_{i=0}^{\deg(p)}\frac{|c_i|}{2^{i}} a^{2i}x^{2i}\bigg)^{2b}\bigg]\\
    &\leq(\deg(p)+2)^{2b}\cdot\max_{i=0}^{\deg(p)}|c_i|^{2b}\cdot\max_{i=0}^{\deg(p)}\frac{a^{4bi}\EE_{x}[x^{4bi}]}{2^{2bi}}\\
    &\leq e^{Cb\sqrt{T\log(1/\epsilon)}}\cdot\max_{i=0}^{\deg(p)}\frac{a^{4bi}\EE_{x}[x^{4bi}]}{2^{2bi}}\\
    &\leq e^{Cb\sqrt{T\log(1/\epsilon)}}\cdot\max_{i=1}^{\deg(p)}\frac{a^{4bi}}{2^{2bi}}\cdot\frac{C'bi\lambda^{4bi}}{1+\alpha}\cdot \Big(\frac{4bi}{1+\alpha}\Big)^{\frac{4bi}{1+\alpha}-\frac{1}{2}}e^{-\frac{4bi}{1+\alpha}+\frac{1+\alpha}{48bi}}\\
    &\leq e^{Cb\sqrt{T\log(1/\epsilon)}+C' b\sqrt{T\log(1/\epsilon)}\cdot\log(ab\lambda\sqrt{T\log(1/\epsilon)})+\frac{1+\alpha}{48b}}\\
    &\leq e^{C'' b(1+\alpha)\sqrt{T\log(1/\epsilon)}\cdot\log(ab\lambda \sqrt{T\log(1/\epsilon)})},
\end{align*}
where the third last inequality is by Lemma~\ref{lm:auxi}. Here $C,C',C''$ are large enough constant.
Putting it all together, we get that
\begin{align*}
    &\EE_{x}\big[(q(x)-e^{-\frac{1}{2}a^2x^2})^{2b}\big]\cdot 2e^{-(\frac{2T}{a^2\lambda^2})^{\frac{1+\alpha}{2}}}\leq 2e^{C'' b\sqrt{T\log(1/\epsilon)}\cdot\log(ab\lambda \sqrt{T\log(1/\epsilon)})-(\frac{2T}{a^2\lambda^2})^{\frac{1+\alpha}{2}}}.
\end{align*}
Choosing
\begin{align*}
    T=\Omega\Big((a^2\lambda^2/2)^{2(1+1/\alpha)}\big(C'''b^2\log^2(ab\lambda\sqrt{\log(1/\epsilon)})\big)^{2/\alpha}\big(\log(1/\epsilon)\big)^{\max\{1,2/\alpha\}}(C''')^{1/\alpha^2}\Big)
\end{align*}
where $C'''$ is a large constant makes the total error less than $2\epsilon$. Since $T$ is $\omega(\log(1/\epsilon))$, the degree of the final polynomial is $O(\sqrt{T\log(1/\epsilon)})$ which is
\begin{align*}
    T=O\Big((a^2\lambda^2/2)^{1+1/\alpha}\big(C'''b^2\log^2(ab\lambda\sqrt{\log(1/\epsilon)})\big)^{1/\alpha}\big(\log(1/\epsilon)\big)^{\max\{1,1/2+1/\alpha\}}(C''')^{1/2\alpha^2}\Big).
\end{align*}
\end{proof}

\subsection{Other Auxiliary Lemmas}

\begin{lemma}\label{lm:auxi3}
Suppose $\xb$ is a distribution on $\{\pm1\}^n$ that is $(\alpha,\lambda)$-strictly subexponential and $\yb$ distributed as $\cN_{1-\rho}(\zb)$. For any fixed $T\subseteq[n]$ and fixed $\zb$, with probability at least $1-\epsilon$ of $\xb$, it holds that
\begin{align*}
    \PP_{\yb}\Big[|\langle\ub_T,\yb_T\rangle|\leq C\cdot\|\ub_T\|_2\Big]\geq1-\epsilon,
\end{align*}
where $\ub=\wb\odot\xb$ and $C=(1-2\rho\sigma)\cdot\lambda\log^{\frac{1}{1+\alpha}}(4/\epsilon)+\sqrt{2\log(2/\epsilon)}$.
\end{lemma}
\begin{proof}
By Hoeffding's inequality, we have
\begin{align*}
    \PP_{\yb}\Big[\big|\langle\ub_{T},\yb_{T}\rangle-\EE_{\yb}[\langle\ub_{T},\yb_{T}\rangle]\big|\geq t\Big]\leq2\exp\Big(-\frac{t^2}{2\|\ub_{T}\|_2^2}\Big).
\end{align*}
Therefore, with probability at least $1-\epsilon$ it holds that
\begin{align*}
    \langle\ub_{T},\yb_{T}\rangle\in\bigg[\EE_{\yb}[\langle\ub_{T},\yb_{T}\rangle]-\sqrt{2\log(2/\epsilon)}\cdot\|\ub_{T}\|_2,\EE_{\yb}[\langle\ub_{T},\yb_{T}\rangle]+\sqrt{2\log(2/\epsilon)}\cdot\|\ub_{T}\|_2\bigg].
\end{align*}
Notice that
\begin{align*}
    \EE_{y_i}[u_i y_i]&=u_i\EE[y_i]=u_i\big((1-\rho)z_i+\rho(1-2\sigma)\big),\\
    \EE_{\yb}[\langle\ub_T,\yb_T\rangle]&=(1-\rho)\sum_{i\in T}u_iz_i+\rho(1-2\sigma)\sum_{i\in T}u_i\\
    &=(1-\rho)\cdot\langle\wb_T\odot\zb_T,\xb_T\rangle+\rho(1-2\sigma)\cdot\langle\wb_T,\xb_T\rangle,
\end{align*}
since $\xb$ is $(\alpha,\lambda)$-strictly subexponential, we have
\begin{align*}
    \PP_{\xb}\Big[|\langle\wb_{T},\xb_{T}\rangle|>t\cdot\|\wb_T\|_2\Big]\leq2\exp\big(-(t/\lambda)^{1+\alpha}\big),
\end{align*}
and
\begin{align*}
    \PP_{\xb}\Big[|\langle\wb_{T}\odot\zb_{T},\xb_{T}\rangle|>t\cdot\|\wb_T\|_2\Big]\leq2\exp\big(-(t/\lambda)^{1+\alpha}\big).
\end{align*}
Thus, with probability at least $1-\epsilon$ of $\xb$ it holds that
\begin{align*}
    |\langle\wb_T,\xb_T\rangle|\leq\lambda\log^{\frac{1}{1+\alpha}}(4/\epsilon)\cdot\|\wb_T\|_2,
\end{align*}
and
\begin{align*}
    |\langle\wb_T\odot\zb_T,\xb_T\rangle|\leq\lambda\log^{\frac{1}{1+\alpha}}(4/\epsilon)\cdot\|\wb_T\|_2.
\end{align*}
Notice that $\xb$ is distributed on $\{\pm1\}^{n}$ and hence $\|\ub_T\|_2=\|\wb_{T}\odot\xb_{T}\|_2=\|\wb_T\|_2$, then with probability at least $1-\epsilon$ of $\yb$ it holds that
\begin{equation}\label{ineq:bound1}
\begin{aligned}
    |\langle\ub_T,\yb_T\rangle|&\leq\big|\EE_{\yb}[\langle\ub_T,\yb_T\rangle]\big|+\sqrt{2\log(2/\epsilon)}\cdot\|\ub_T\|_2\\
    &\leq(1-\rho)\cdot|\langle\wb_T\odot\zb_T,\xb_T\rangle|+\rho(1-2\sigma)\cdot|\langle\wb_T,\xb_T\rangle|+\sqrt{2\log(2/\epsilon)}\cdot\|\ub_T\|_2\\
    &\leq\big((1-\rho)+\rho(1-2\sigma)\big)\cdot\lambda\log^{\frac{1}{1+\alpha}}(4/\epsilon)\cdot\|\wb_T\|_2+\sqrt{2\log(2/\epsilon)}\cdot\|\ub_{T}\|_2\\
    &\leq(1-2\rho\sigma)\cdot\lambda\log^{\frac{1}{1+\alpha}}(4/\epsilon)\cdot\|\wb_T\|_2+\sqrt{2\log(2/\epsilon)}\cdot\|\ub_{T}\|_2\\
    &=C\cdot\|\ub_T\|_2,
\end{aligned}
\end{equation}
where $C=(1-2\rho\sigma)\cdot\lambda\log^{\frac{1}{1+\alpha}}(4/\epsilon)+\sqrt{2\log(2/\epsilon)}$. 
\end{proof}

\begin{lemma}\label{lm:regularity}
Suppose that the vector $\wb\in\RR^{n}$ is $\alpha$-regular, i.e., $\|\wb\|_{\infty}\leq\alpha\cdot\|\wb\|_{2}$. Suppose $\ub$ is a $n$-dimensional random vector where each coordinate is 1 with probability $\rho$ and 0 with probability $1-\rho$ independently. If $\alpha\leq\rho/\sqrt{\log(1/\delta)/2}$, then with probability at least $1-\delta$ over the randomness of $\ub$ it holds that
\begin{align*}
    \|\wb\|_{\infty}\leq\alpha\cdot\|\wb\|_2\leq c\cdot\alpha\cdot\|\wb\odot\ub\|_{2},
\end{align*}
where $c=(\rho-\sqrt{\log(1/\delta)/2}\cdot\alpha)^{-\frac{1}{2}}$. If condition $\alpha\leq\rho/\sqrt{2\log(1/\delta)}$ holds, then with probability at least $1-\delta$ over the randomness of $\ub$ it holds that
\begin{align*}
    \|\wb\|_{\infty}\leq \alpha\cdot\|\wb\|_{2}\leq (\rho/2)^{-\frac{1}{2}}\cdot\alpha\cdot\|\wb\odot\ub\|_{2}.
\end{align*}
\end{lemma}

\begin{proof}
By Hoeffding's inequality and notice that $\EE[(w_i u_i)^2]=w_i^2\cdot\EE[u_i]=\rho w_i^2$ and $0\leq (w_i u_i)^2\leq w_i^2$, we obtain
\begin{align*}
    \PP\Big(\|\wb\odot\ub\|_2^2-\rho\cdot\|\wb\|_2^2\leq -t\Big)&\leq\exp\bigg(-\frac{2t^2}{\|\wb\|_4^4}\bigg)\\
    &\leq\exp\bigg(-\frac{2t^2}{\|\wb\|_{\infty}^{2}\|\wb\|_2^2}\bigg)\\
    &\leq\exp\bigg(-\frac{2t^2}{\alpha^2\cdot\|\wb\|_2^4}\bigg),
\end{align*}
where the second inequality is by the definition of the infinity norm, and the third inequality is by the $\alpha$-regularity condition. Therefore, with probability at least $1-\delta$, we can get
\begin{align*}
    \|\wb\odot\ub\|_2^2\geq\rho\cdot\|\wb\|_2^2-\sqrt{\frac{\log(1/\delta)}{2}}\cdot\alpha\cdot\|\wb\|_2^2.
\end{align*}
Then, with probability at least $1-\delta$, it holds that
\begin{align*}
    \|\wb\|_{\infty}&\leq\alpha\cdot\|\wb\|_{2}\leq\alpha\cdot\bigg(\rho-\sqrt{\frac{\log(1/\delta)}{2}}\cdot\alpha\bigg)^{-\frac{1}{2}}\cdot\|\wb\odot\ub\|_2.
\end{align*}

\end{proof}

\begin{theorem}[Berry-Esseen CLT]\label{thm:Berry-Esseen}
Let $X_1, X_2, ...,$ be independent random variables with $\EE[X_i]=0,\EE[X_i^2]=\sigma_i^2>0,$ and $\EE[|X_i|^3]=\rho_i<\infty$. Also, let
\begin{align*}
    S_n=\frac{X_1+X_2+\cdots+X_n}{\sqrt{\sigma_1^2+\sigma_2^2+\cdots+\sigma_n^2}}
\end{align*}
be the normalized $n$-th partial sum. Denote $F_n$ the cdf of $S_n$, and $\Phi$ the cdf of the standard normal distribution. Then for all $n$ there exists an absolute constant $C$ such that
\begin{align*}
    \sup_{x\in\RR}|F_n(x)-\Phi(x)|\leq C\cdot\Big(\sum_{i=1}^{n}\sigma_i^2\Big)^{-\frac{3}{2}}\cdot\sum_{i=1}^{n}\rho_i.
\end{align*}
\end{theorem}

\begin{lemma}[Lemma D.11 in \cite{chandrasekaran2024smoothed}]\label{lm:auxi2}
For $T>0$ and error $\epsilon>0$, there exists a polynomial $p$ such that
\begin{enumerate}
    \item $\sup_{x\in[0,T]}|p(x)-e^{-x}|\leq\epsilon$.
    \item $\deg(p)\leq O(\sqrt{T\log(1/\epsilon)})$, if $T=\omega(\log(1/\epsilon))$.
    \item $p(x)=\sum_{i=1}^{\deg(p)}c_i x^i$ where $|c_i|\leq e^{C\sqrt{T\log(1/\epsilon)}}$ for all $i\leq\deg(p)$. Here $C$ is a large enough constant.
\end{enumerate}
\end{lemma}

\begin{lemma}\label{lm:auxi}
If $x$ is $(\alpha,\lambda)$-strictly sub-exponential random variable satisfying $\PP_{x}[|x|>t]\leq2\cdot e^{-(t/\lambda)^{1+\alpha}}$, then the $k$-th moment is upper bounded by:
\begin{align*}
    E_{x}[|x|^{k}]\leq C\lambda^{k}\cdot \Big(\frac{k}{1+\alpha}\Big)^{\frac{k}{1+\alpha}+\frac{1}{2}}e^{-\frac{k}{1+\alpha}+\frac{1+\alpha}{12k}},
\end{align*}
where $C$ is a universal constant.
\end{lemma}

\begin{proof}
By the layer cake representation and the tail bound, we have
\begin{align*}
    E_{x}[|x|^{k}]&=k\int_{0}^{\infty}t^{k-1}\cdot \PP_{x}[|x|>t]dt\leq 2k\int_{0}^{\infty}t^{k-1}\cdot e^{-(t/\lambda)^{1+\alpha}}dt.
\end{align*}
Making the substitution $s=(t/\lambda)^{1+\alpha}$, i.e. $t=\lambda s^{\frac{1}{1+\alpha}}$ and $dt=\frac{\lambda}{1+\alpha}s^{-\frac{\alpha}{1+\alpha}}ds$, we get
\begin{align*}
    E_{x}[|x|^{k}]&\leq 2k\int_{0}^{\infty}\lambda^{k-1}s^{\frac{k-1}{1+\alpha}}\cdot e^{-s}\cdot\frac{\lambda}{1+\alpha}s^{-\frac{\alpha}{1+\alpha}} ds\\
    &=\frac{2k\lambda^k}{1+\alpha}\int_{0}^{\infty}s^{\frac{k}{1+\alpha}-1}\cdot e^{-s} ds\\
    &=\frac{2k\lambda^k}{1+\alpha}\cdot\Gamma\Big(\frac{k}{1+\alpha}\Big).
\end{align*}
Notice that $\Gamma(x)\leq Cx^{x-1/2}e^{-x}e^{1/(12x)}$ for a positive constant $C$, then we have
\begin{align*}
    E_{x}[|x|^{k}]\leq \frac{2k\lambda^k}{1+\alpha}\cdot C\Big(\frac{k}{1+\alpha}\Big)^{\frac{k}{1+\alpha}-\frac{1}{2}}e^{-\frac{k}{1+\alpha}+\frac{1+\alpha}{12k}}.
\end{align*}
\end{proof}

\begin{lemma}\label{lm:auxi1}
If $x$ is $(\alpha,\lambda)$-strictly sub-exponential random variable satisfying $\PP_{x}[|x|>t]\leq2\cdot e^{-(t/\lambda)^{1+\alpha}}$, then for any $a>0$
\begin{align*}
    \EE_{x}[e^{a|x|}]\leq 3e^{2^{1/\alpha}(a\lambda)^{1+1/\alpha}}.
\end{align*}
\end{lemma}
\begin{proof}
We split the integral into two parts at some threshold $T$, which we choose later:
\begin{align*}
    \EE_{x}[e^{a|x|}]&=\EE_{x}\big[e^{a|x|}\cdot\ind[|x|<T]\big]+\EE_{x}\big[e^{a|x|}\cdot\ind[|x|\geq T]\big]\\
    &\leq e^{aT}+\EE_{x}\big[e^{a|x|}\cdot\ind[|x|\geq T]\big].
\end{align*}
By the layer cake representation and the tail bound, we have
\begin{align*}
    \EE_{x}\big[e^{a|x|}\cdot\ind[|x|\geq T]\big]=\int_{T}^{\infty}ae^{at}\cdot\PP_{x}[|x|>t]dt\leq 2a\int_{T}^{\infty}e^{at}\cdot e^{-(t/\lambda)^{1+\alpha}}dt.
\end{align*}
Choose $T$ to be $(2a)^{1/\alpha}\lambda^{1+1/\alpha}$. Then, we have $(t/\lambda)^{1+\alpha}\geq 2at$ for $t\geq T$ and hence
\begin{align*}
    \EE_{x}\big[e^{a|x|}\cdot\ind[|x|\geq T]\big]&\leq 2a\int_{T}^{\infty}e^{-at}dt= 2e^{-aT}.
\end{align*}
Thus,
\begin{align*}
    \EE_{x}[e^{a|x|}]\leq e^{aT}+2e^{-aT}\leq 3e^{aT}=3e^{2^{1/\alpha}(a\lambda)^{1+1/\alpha}}.
\end{align*}

\end{proof}

% \section{Technical Appendices and Supplementary Material}
% Technical appendices with additional results, figures, graphs and proofs may be submitted with the paper submission before the full submission deadline (see above), or as a separate PDF in the ZIP file below before the supplementary material deadline. There is no page limit for the technical appendices.

%%%%%%%%%%%%%%%%%%%%%%%%%%%%%%%%%%%%%%%%%%%%%%%%%%%%%%%%%%%%

\newpage
\section*{NeurIPS Paper Checklist}

%%% BEGIN INSTRUCTIONS %%%
The checklist is designed to encourage best practices for responsible machine learning research, addressing issues of reproducibility, transparency, research ethics, and societal impact. Do not remove the checklist: {\bf The papers not including the checklist will be desk rejected.} The checklist should follow the references and follow the (optional) supplemental material.  The checklist does NOT count towards the page
limit. 

Please read the checklist guidelines carefully for information on how to answer these questions. For each question in the checklist:
\begin{itemize}
    \item You should answer \answerYes{}, \answerNo{}, or \answerNA{}.
    \item \answerNA{} means either that the question is Not Applicable for that particular paper or the relevant information is Not Available.
    \item Please provide a short (1–2 sentence) justification right after your answer (even for NA). 
   % \item {\bf The papers not including the checklist will be desk rejected.}
\end{itemize}

{\bf The checklist answers are an integral part of your paper submission.} They are visible to the reviewers, area chairs, senior area chairs, and ethics reviewers. You will be asked to also include it (after eventual revisions) with the final version of your paper, and its final version will be published with the paper.

The reviewers of your paper will be asked to use the checklist as one of the factors in their evaluation. While "\answerYes{}" is generally preferable to "\answerNo{}", it is perfectly acceptable to answer "\answerNo{}" provided a proper justification is given (e.g., "error bars are not reported because it would be too computationally expensive" or "we were unable to find the license for the dataset we used"). In general, answering "\answerNo{}" or "\answerNA{}" is not grounds for rejection. While the questions are phrased in a binary way, we acknowledge that the true answer is often more nuanced, so please just use your best judgment and write a justification to elaborate. All supporting evidence can appear either in the main paper or the supplemental material, provided in appendix. If you answer \answerYes{} to a question, in the justification please point to the section(s) where related material for the question can be found.

IMPORTANT, please:
\begin{itemize}
    \item {\bf Delete this instruction block, but keep the section heading ``NeurIPS Paper Checklist"},
    \item  {\bf Keep the checklist subsection headings, questions/answers and guidelines below.}
    \item {\bf Do not modify the questions and only use the provided macros for your answers}.
\end{itemize}

%%% END INSTRUCTIONS %%%

\begin{enumerate}

\item {\bf Claims}
    \item[] Question: Do the main claims made in the abstract and introduction accurately reflect the paper's contributions and scope?
    \item[] Answer: \answerYes{} % Replace by \answerYes{}, \answerNo{}, or \answerNA{}.
    \item[] Justification: The abstract and introduction clearly state the paper's main contribution: a discrete analogue of smoothed agnostic learning for Boolean halfspaces under strictly sub-exponential distributions, along with a computationally efficient learning algorithm. These claims are supported by the theoretical results in the main body (Section~\ref{sec:results}).

\item {\bf Limitations}
    \item[] Question: Does the paper discuss the limitations of the work performed by the authors?
    \item[] Answer: \answerYes{} % Replace by \answerYes{}, \answerNo{}, or \answerNA{}.
    \item[] Justification: The limitations of our work are discussed in the conclusion section of this paper (Section~\ref{sec:conclusion}).

\item {\bf Theory assumptions and proofs}
    \item[] Question: For each theoretical result, does the paper provide the full set of assumptions and a complete (and correct) proof?
    \item[] Answer: \answerYes{} % Replace by \answerYes{}, \answerNo{}, or \answerNA{}.
    \item[] Justification: All theoretical results are clearly stated with assumptions, and full proofs are provided in the main paper or the supplementary material. Lemmas and tools from prior work are cited and contextualized.

    \item {\bf Experimental result reproducibility}
    \item[] Question: Does the paper fully disclose all the information needed to reproduce the main experimental results of the paper to the extent that it affects the main claims and/or conclusions of the paper (regardless of whether the code and data are provided or not)?
    \item[] Answer: \answerNA{} % Replace by \answerYes{}, \answerNo{}, or \answerNA{}.
    \item[] Justification: The paper is entirely theoretical and does not include empirical experiments.

\item {\bf Open access to data and code}
    \item[] Question: Does the paper provide open access to the data and code, with sufficient instructions to faithfully reproduce the main experimental results, as described in supplemental material?
    \item[] Answer: \answerNA{} % Replace by \answerYes{}, \answerNo{}, or \answerNA{}.
    \item[] Justification: No datasets or code are used or required, as the paper does not include experiments.

\item {\bf Experimental setting/details}
    \item[] Question: Does the paper specify all the training and test details (e.g., data splits, hyperparameters, how they were chosen, type of optimizer, etc.) necessary to understand the results?
    \item[] Answer: \answerNA{} % Replace by \answerYes{}, \answerNo{}, or \answerNA{}.
    \item[] Justification: No experimental setting is involved; the paper is theoretical.

\item {\bf Experiment statistical significance}
    \item[] Question: Does the paper report error bars suitably and correctly defined or other appropriate information about the statistical significance of the experiments?
    \item[] Answer: \answerNA{} % Replace by \answerYes{}, \answerNo{}, or \answerNA{}.
    \item[] Justification: The paper does not include empirical results or plots requiring error bars or statistical tests.

\item {\bf Experiments compute resources}
    \item[] Question: For each experiment, does the paper provide sufficient information on the computer resources (type of compute workers, memory, time of execution) needed to reproduce the experiments?
    \item[] Answer: \answerNA{} % Replace by \answerYes{}, \answerNo{}, or \answerNA{}.
    \item[] Justification: No experiments were conducted; no compute resources were used.
    
\item {\bf Code of ethics}
    \item[] Question: Does the research conducted in the paper conform, in every respect, with the NeurIPS Code of Ethics \url{https://neurips.cc/public/EthicsGuidelines}?
    \item[] Answer: \answerYes{} % Replace by \answerYes{}, \answerNo{}, or \answerNA{}.
    \item[] Justification: The research is theoretical, complies with NeurIPS ethical guidelines, and does not raise ethical concerns regarding data use, privacy, or fairness.

\item {\bf Broader impacts}
    \item[] Question: Does the paper discuss both potential positive societal impacts and negative societal impacts of the work performed?
    \item[] Answer: \answerNA{} % Replace by \answerYes{}, \answerNo{}, or \answerNA{}.
    \item[] Justification: The paper is purely theoretical.
    
\item {\bf Safeguards}
    \item[] Question: Does the paper describe safeguards that have been put in place for responsible release of data or models that have a high risk for misuse (e.g., pretrained language models, image generators, or scraped datasets)?
    \item[] Answer: \answerNA{} % Replace by \answerYes{}, \answerNo{}, or \answerNA{}.
    \item[] Justification: The paper does not release data, models, or tools with potential misuse risk.

\item {\bf Licenses for existing assets}
    \item[] Question: Are the creators or original owners of assets (e.g., code, data, models), used in the paper, properly credited and are the license and terms of use explicitly mentioned and properly respected?
    \item[] Answer: \answerNA{} % Replace by \answerYes{}, \answerNo{}, or \answerNA{}.
    \item[] Justification: No external datasets, models, or software assets are used.

\item {\bf New assets}
    \item[] Question: Are new assets introduced in the paper well documented and is the documentation provided alongside the assets?
    \item[] Answer: \answerNA{} % Replace by \answerYes{}, \answerNo{}, or \answerNA{}.
    \item[] Justification: No new datasets or software assets are introduced.

\item {\bf Crowdsourcing and research with human subjects}
    \item[] Question: For crowdsourcing experiments and research with human subjects, does the paper include the full text of instructions given to participants and screenshots, if applicable, as well as details about compensation (if any)? 
    \item[] Answer: \answerNA{} % Replace by \answerYes{}, \answerNo{}, or \answerNA{}.
    \item[] Justification: No human subjects or crowdsourcing are involved.

\item {\bf Institutional review board (IRB) approvals or equivalent for research with human subjects}
    \item[] Question: Does the paper describe potential risks incurred by study participants, whether such risks were disclosed to the subjects, and whether Institutional Review Board (IRB) approvals (or an equivalent approval/review based on the requirements of your country or institution) were obtained?
    \item[] Answer: \answerNA{} % Replace by \answerYes{}, \answerNo{}, or \answerNA{}.
    \item[] Justification: No IRB approval is necessary, as no human subjects are involved.

\item {\bf Declaration of LLM usage}
    \item[] Question: Does the paper describe the usage of LLMs if it is an important, original, or non-standard component of the core methods in this research? Note that if the LLM is used only for writing, editing, or formatting purposes and does not impact the core methodology, scientific rigorousness, or originality of the research, declaration is not required.
    %this research? 
    \item[] Answer: \answerNA{} % Replace by \answerYes{}, \answerNo{}, or \answerNA{}.
    \item[] Justification: No LLMs were used in developing or supporting the core scientific contributions of this research.

\end{enumerate}

\end{document}